\documentclass[11pt]{article}

\title{Suboptimal Performance of the Bayes Optimal Algorithm in Frequentist Best Arm Identification}
\author{Junpei Komiyama\thanks{junpei@komiyama.info}}
\date{\today}

\usepackage{natbib}
\usepackage{graphicx}
\usepackage[driver=dvipdfm,margin=0.8in]{geometry}
\usepackage{setspace}
\usepackage{color}
\usepackage{amsmath,amssymb,amsthm}
\usepackage[at]{easylist}

\newcommand{\replaced}[2]{#2}
\newcommand{\added}[1]{#1}

\usepackage{mathtools}
\usepackage{amsmath}
\usepackage{amsthm}
\usepackage{amssymb}
\usepackage{setspace}
\usepackage{booktabs}
\usepackage{tabularx}

\mathtoolsset{showonlyrefs}

\makeatletter

\usepackage{amsthm}
\usepackage{amsfonts}
\usepackage{bbm}
\usepackage{graphics}
\usepackage{enumerate}
\usepackage{float}
\usepackage{caption}
\usepackage{subcaption}

\usepackage[autostyle]{csquotes}

\theoremstyle{remark}

\theoremstyle{definition}
\newtheorem{thm}{Theorem}
\newtheorem{lem}[thm]{Lemma}
\newtheorem{definition}{Definition}

\newtheorem{remark}{Remark}

\newtheorem{conjecture}[thm]{Conjecture}

\usepackage{amsmath} 
\usepackage{amssymb}
\usepackage{bm}
\usepackage{ascmac}
\usepackage{setspace}
\usepackage[at]{easylist}

\usepackage[hyphens]{url}
\usepackage[colorlinks,urlcolor=blue, citecolor=blue, menucolor=blue]{hyperref}
\theoremstyle{plain}

\usepackage{algorithm,algpseudocode}

\usepackage{amsmath}
\DeclareMathOperator*{\argmax}{arg\,max}
\DeclareMathOperator*{\argmin}{arg\,min}

\newcommand{\BI}{B}
\newcommand{\Real}{\mathbb{R}}
\newcommand{\Natural}{\mathbb{N}}

\newcommand{\SRegBayes}{\mathrm{R}_{\bH}}

\newcommand{\SRegBayesi}[2]{\mathrm{R}_{#1}^{(#2)}}

\newcommand{\SRegFreq}{\mathrm{R}_{\bmu}}
\newcommand{\SRegFreqTrue}{\mathrm{R}_{\bmutrue}}

\newcommand{\hDelta}{\hat{\Delta}} %
\newcommand{\hDeltarec}{\hat{\Delta}^{\mathrm{r}}} %

\newcommand{\Ex}{\mathbb{E}}
\newcommand{\Prob}{\mathbb{P}}

\newcommand{\Normal}{\mathcal{N}}

\newcommand{\Ind}{\bm{1}}

\newcommand{\ist}{i^{*}}
\newcommand{\must}{\mu^{*}}

\newcommand{\Vrest}{V_r}
\newcommand{\Vlast}{V_l}
\newcommand{\Vgap}{V_g}
\newcommand{\Vtrans}{V_s}

\newcommand{\EA}{\mathcal{A}}
\newcommand{\EB}{\mathcal{B}}

\newcommand{\EF}{\mathcal{F}}

\newcommand{\EW}{\mathcal{W}}
\newcommand{\EX}{\mathcal{X}}
\newcommand{\EY}{\mathcal{Y}}
\newcommand{\EZ}{\mathcal{Z}}

\newcommand{\bmu}{\bm{\mu}}
\newcommand{\bmutrue}{\bm{\mu}^{(0)}}
\newcommand{\mutrue}{\mu^{(0)}}

\newcommand{\bH}{\bm{H}}
\newcommand{\Srec}{S^{\mathrm{r}}}
\newcommand{\Jrec}{J^{\mathrm{r}}}
\newcommand{\SPrec}{S^{\mathrm{r}'}}
\newcommand{\Nrec}{N^{\mathrm{r}}}
\newcommand{\NPrec}{N^{\mathrm{r}'}}
\newcommand{\Urec}{U^{\mathrm{r}}}
\newcommand{\Xrec}{X^{\mathrm{r}}}
\newcommand{\Irec}{I^{\mathrm{r}}}
\newcommand{\bHrec}{\bm{H}^{\mathrm{r}}}
\newcommand{\Hrec}{H^{\mathrm{r}}}

\newcommand{\hatmu}{\hat{\mu}}
\newcommand{\hatmurec}{\hat{\mu}^{\mathrm{r}}}

\newcommand{\tilmu}{\tilde{\mu}}

\newcommand{\poly}{\mathrm{poly}}

\newcommand{\Hexrec}[1]{\bm{H}_{\setminus #1}^{\mathrm{r}}}

\newcommand{\Cbayesopt}{C_{\mathrm{bayes}}}

\newcommand{\ABayesOpt}{A^*}

\newcommand{\AFreq}{A^{\mathrm{freq}}}
\newcommand{\CUnder}{C_U}
\newcommand{\Lmock}{L^{\mathrm{mock}}}

\newcommand{\gthree}{g_0}

\newcommand{\funder}{f_1} 
\newcommand{\fclose}{f_2} 
\newcommand{\fnodrift}{f_3} 
\newcommand{\Gap}{\Delta_G}

\allowdisplaybreaks

\makeatother

\begin{document}

\maketitle

\onehalfspacing

\begin{abstract}
We consider the fixed-budget best arm identification problem with \replaced{Normal reward distributions}{rewards following normal distributions}. In this problem, the forecaster is given $K$ arms (or treatments) and $T$ time steps. The forecaster attempts to find the arm with the largest mean, via an adaptive experiment conducted using an algorithm. The algorithm's performance is \replaced{measured by the simple regret, that is, the quality of the estimated best arm.}{evaluated by ``simple regret'', reflecting the quality of the estimated best arm.}
\replaced{The frequentist simple regret can be exponentially small to $T$, whereas the Bayesian simple regret is polynomially small to $T$.}{While frequentist simple regret can decrease exponentially with respect to $T$, Bayesian simple regret decreases polynomially.} This paper demonstrates that the Bayes optimal algorithm, which minimizes the Bayesian simple regret, does not \replaced{produce an exponential simple regret for some parameters, a finding that}{yield an exponential decrease in simple regret under certain parameter settings. This} contrasts with the \replaced{many results indicating the}{numerous findings that suggest the} asymptotic equivalence of Bayesian and frequentist \replaced{algorithms}{approaches} in\replaced{ the context of}{} fixed sampling regimes. \replaced{While the Bayes optimal algorithm is described in terms of a recursive equation that is virtually impossible to compute exactly, we establish the foundations for further analysis by introducing a key quantity that we call the expected Bellman improvement.}{Although the Bayes optimal algorithm is formulated as a recursive equation that is virtually impossible to compute exactly, we lay the groundwork for future research by introducing a novel concept termed the \textit{expected Bellman improvement}.}
\end{abstract}

\section*{Quote}

The following paragraph appears in a conversation between Joseph Leo Doob and James Laurie Snell \citep{Snell1997}.
\begin{displayquote}
While writing my book I (Doob) had an argument with Feller. He asserted that everyone said ``random variable'' and I asserted that everyone said ``chance variable.'' We obviously had to use the same name in our books, so we decided the issue by a stochastic procedure. That is, we tossed for it and he won.
\end{displayquote}
The coin that Doob and Feller tossed separated our world, where we discuss the randomness of random variables, from a parallel world where we discuss the chanciness of chance variables. 
In this paper, we consider a Bayes optimal algorithm, which computes the consequences over all the exponentially many parallel worlds. 

\section{Introduction} \label{sec_prob}

This paper addresses the problem of identifying the best arm (or treatment) among multiple options from a fixed number of samples.
This problem associates each arm with an (unknown) parameterized distribution, which is a (noisy) signal of the quality of the treatment.
To maximize the overall effectiveness of the treatments, the forecaster uses an algorithm that adaptively determines which arms to select. 
Such adaptive allocation is essential for solving many practical problems, including A/B testing in online platforms \citep{li2010}, simulation optimization \citep{MalkomesCLM21}, and clinical trials \citep{Thompson1933,villar2015}.
\replaced{Several classes of algorithms have been previously considered, among which}{Among the best-known algorithms for sequential selection are the multi-armed bandit algorithms}, such as Thompson sampling \citep{Thompson1933}, and upper confidence bounds \citep{Lairobbins1985,auer2002}. Although multi-armed bandit algorithms solve the exploration--exploitation tradeoff to maximize total effectiveness during the rounds, they \replaced{are suboptimal with respect to their statistical power to detect the best-quality arm.}{have limited ability to detect the best arm \citep{bubeck_pure_2009}, and are thus suboptimal for our aim.} 

The machine learning literature refers to the process of finding the best-quality arm as best arm identification (BAI; \citealt{Audibert10}).
Unlike multi-armed bandit algorithms, BAI algorithms are designed solely to deliver the most effective exploration.\footnote{Although the term ``best arm identification'' has appeared only recently, several strands of research share the same goal, among which ranking and selection (RS; \citealt{Bechhofer1954}) is among the best known. Several works, such as \cite{Russo2020,Hong2021review}, consider the difference between BAI and RS, and we briefly discuss work related to RS in Section~\ref{sec_related}. However, because BAI and RS share several overlapping ideas, we do not explicitly outline the difference.}
\added{BAI has two settings, namely, the fixed-budget setting which the number of samples is fixed, and the fixed-confidence setting in which the forecaster stops when the confidence level on the best arm reaches a predefined threshold. In this paper, we consider the former setting.}
\replaced{Most studies on BAI use frequency evaluations, especially those comparing the algorithm's performance in the worst-case scenario. However, many algorithms are Bayesian, including}{Popular algorithms for fixed-budget identification include successive rejects (SR; \citealt{Audibert10}) and successive halving \citep{Karnin2013}. There are also several Bayesian algorithms that utilize a prior, such as} top-two Thompson sampling \citep{Russo2020}, knowledge gradient (KG; \citealt{Gupta1994}), and expected improvement (EI; \citealt{Jones1998}). A Bayesian algorithm is initialized with a prior belief, and the forecaster learns from each sample to build a posterior belief. Given that, we know the dynamics of the posterior belief as a distribution, exact \replaced{optimization of Bayesian BAI}{minimization of the Bayes risk} can be formulated as dynamic programming. Nonetheless, exact dynamic programming solutions are computationally intractable because of the curse of dimensionality \citep{Russo2020,powell2013optimal}. Consequently, KG and EI are designed to provide the one-step lookahead approximations of an exactly optimal dynamic programming solution.

\added{While some classes of Bayesian algorithms are known to perform optimally in the fixed-confidence setting \citep{Xuedong2020}, limited results are known regarding their performance in the fixed-budget setting.}
This paper considers \replaced{whether access to infinite computational power to exactly solve the dynamic programming would connote a Bayes optimal algorithm that performs well not only in the Bayesian sense but also in the frequentist sense.}{the Bayes optimal algorithm in the context of fixed-budget identification. 
In the field of statistical decision theory, the Bayes optimal algorithm is deeply connected to a minimax estimator that maximizes the worst-case performance. 
We are interested in the following research question: Does the Bayes optimal algorithm have a good performance
in the frequentist measure? 
} 
On the contrary, we show that the Bayes optimal algorithm performs sub-optimally with some of the worst model parameters, which implies that maximizing the Bayesian objective differs substantially from maximizing the frequentist objective. \added{A Bayesian measure requires it to optimize the performance} \replaced{on average}{averaged over the prior}, \added{whereas a frequentist measure requires an algorithm to optimize the performance against any case.} \replaced{The suboptimality of the Bayes optimal algorithm implies that, even if we extend KG and EI to look beyond one step ahead, their performance does not necessarily improve in terms of the frequentist.}{The fact that the Bayes optimal algorithm is suboptimal means that even if we enhance KG and EI to plan more than one step ahead, their performance in a frequentist measure might not improve.}

Note that this discrepancy between Bayesian and frequentist \added{measures} differs considerably from the situation in standard statistical inference with non-adaptive samples. For a fixed-sample problem, the Bernstein–von Mises theorem describes the asymptotic equivalence of Bayesian and frequentist inference. \replaced{However, in an adaptive sampling scheme, randomness in early rounds affects the number of samples in subsequent rounds. The more severely an arm is underestimated, the smaller the number of subsequent samples allocated to the arm.}{However, in an adaptive sampling scheme, underestimation of an arm due to the randomness of the empirical mean results in a smaller number of samples of the arm in the future.} Bayesian and frequentist BAI algorithms are both robust against \replaced{such an underestimation}{such randomness} but with different confidence levels. Bayesian algorithms are robust up to a polynomially small underestimation, whereas frequentist algorithms are robust up to an exponentially small underestimation.

\replaced{Our results are analytically innovative because despite being restricted to a specific instance with a limited number of arms (namely, a three-armed problem), they represent the foundations for the formal analysis of an exact solution to the dynamic programming that is not limited to a one- (or two-)step lookahead. We present several cases where the \replaced{draw of arms is computable}{the exact analysis of the dynamic programming is feasible} despite the need to calculate the evolution of the posterior in a very long future round.}{Our results offer analytical innovation by establishing foundational principles for a formal analysis that yields exact solutions in dynamic programming. This is notable because it enables solutions that extend beyond the conventional one- or two-step lookahead. We demonstrate several instances in which it is feasible to perform exact analyses of dynamic programming  regardless of the need to project the evolution of the posterior over extended future periods}

The structure of the paper is as follows. The rest of this section establishes the problem. Section \ref{sec_boa}~formalizes the Bayes optimal algorithm. Section~\ref{sec_main} describes our main results, and \added{Sections~\ref{sec_proof}--\ref{sec_proofsketch} describes their derivation}. 
Section~\ref{sec_sim_brief} presents the results of an auxiliary simulation. 
Section~\ref{sec_related} discusses related work, and Section~\ref{sec_disc} concludes the paper. Most details of the proofs appear in the appendices.

\subsection{Problem setup}
\label{sec_setup}

We consider the normal BAI problem with $K$ arms and $T$ samples. 
Each arm $i \in [K] = \{1,2,\dots,K\}$ is associated with an (unknown) parameter $\mu_i \in \Real$. 
We use $\bmu = (\mu_1,\mu_2,\dots,\mu_K)$ to denote the set of parameters associated with the arms. The problem involves $T$ rounds, and at each round $t=1,2,\dots,T$, the forecaster can choose one of the arms. Upon choosing arm $i$, the forecaster observes reward $X(t) \sim \Normal(\mu_i, 1)$, which represents the effectiveness of the selected arm. \replaced{}{Here, $\Normal(\mu_i, 1)$ is a normal distribution with mean $\mu_1$ and unit variance.}

The goal of the forecaster is to identify the best arm $\ist = \argmax_i \mu_i$, defined as the arm with the largest mean $\must$, using an effective algorithm.
Any such algorithm comprises a sampling rule $I(1), I(2), I(3), \dots, I(T)$ and a recommendation rule $J(T)$.
Let $\EF_t$ be the $\sigma$-algebra generated by 
\[
(U(0), I(1), X(1), U(1), I(2), X(2), \dots, U(t-1), I(t), X(t)).
\]
Then, the sampling rule $I(t)$ is $\EF_{t-1}$-measurable, meaning that the algorithm chooses an arm at round $t$ according to the observation history $\{I(s), X(s)\}_{s<t}$ and certain external randomizers $\{U(s)\}_{s<t}$.\footnote{\replaced{When a deterministic algorithm is used, it does not depend on $\{U(s)\}_{s<t}$.}{}}
\added{We may assume $U(s)$ to be a random variable drawn from the standard normal distribution $\Normal(0,1)$. An algorithm is deterministic when it does not use the information of $\{U(s)\}_{s<t}$.}

The recommendation rule is used to choose an arm $J(T) \in [K]$ that is $\EF_T$-measurable. The performance of the forecaster, which follows algorithm $A$, is measured by the simple regret
\begin{equation}\label{ineq_freqsreg}
\SRegFreq(T, A) = \must - \Ex_{\bmu}[\mu_{J(T)}],
\end{equation}
which is the expected difference between the means of the best arm and the recommended arm $J(T)$. 
We omit $A$ from $\SRegBayes(T, A)$ when it is clear which algorithm is used.

The simple regret formulated by Eq.~\eqref{ineq_freqsreg} is frequentist in the sense that it assumes that $\bmu$ is fixed (but unknown). In contrast, a Bayesian \added{approach} considers a \added{known} distribution of $\bmu \in \Real^K$ and evaluates an algorithm with respect to these expected values. 
More concretely, the Bayesian simple regret is given by
\begin{align}\label{def_reg_bayes}
\SRegBayes(T, A) 
= \Ex_{\bH}[ \SRegFreq(T, A) ]
:= \int \SRegFreq(T, A) d\bH(\bmu),
\end{align}
which marginalizes $\SRegFreq(T, A)$ over the prior $\bH = \bH(\bmu)$.

\subsection{Convergence rates: Frequentist simple regret compared with Bayesian simple regret}
\label{subsec_conv}

\added{A good algorithm, from the perspective of frequentist measure, achieves exponentially small regret as a function of $T$ for any distribution.
\begin{definition}{\rm (Frequentist rate)}
Let $\Gamma(\bmu)$ be any non-negative function that takes a bounded positive value for any $\bmu$ with a unique best arm.
A BAI algorithm has a frequentist rate $\Gamma(\bmu)$ if, for any distribution $\bmu$ with a unique best arm, it holds that
\begin{equation}\label{ineq_freqreg}
\SRegFreq(T) \le e^{-\Gamma(\bmu) T + o(T)},
\end{equation}
where $o(T)$ is a function of $K, \bmu,$ and $T$ that grows more slowly than $T$ when we view $K$ and $\bmu$ as constants.
\end{definition}
}

\replaced{Given fixed $\bmu$, the frequentist simple regret decays exponentially to $T$.}{} For example, the successive rejects algorithm \citep{Audibert10} has a frequentist rate $\Gamma(\bmu) = O\left( (1/\log K)\left(\sum_{i: \mu_i \ne \must} 1/(\must - \mu_i)^2\right)^{-1} \right)$. 
\added{
The uniform algorithm, which draws all arms uniformly, has a rate $\Gamma(\bmu) = \min_i (\mu^* - \mu_i)^2 / K$, which is $\tilde{O}(K)$ times worse than that of SR.
}
More recently, \cite{komiyama2022fbbai} derived the possible value of the rate $\Gamma(\bmu)$.

Meanwhile, under some continuity conditions, the Bayesian simple regret is $\Theta(1/T)$. 
More concretely, \cite{komiyama2021} showed that
\begin{equation}\label{ineq_bayessreg}
\SRegBayes(T) \le \frac{\Cbayesopt(\bH)}{T} + o\left(\frac{1}{T}\right),
\end{equation}
where the value $\Cbayesopt(\bH)$ depends on the prior $\bH$. 

\begin{remark}{\rm (Reward distribution)}
Although the frequentist results of \cite{Audibert10} and \cite{Carpentier2016} concern Bernoulli arms, the same results can be applied to Gaussian arms with unit variance because the empirical means of both the Bernoulli and Gaussian distributions are bounded by the Hoeffding inequality, upon which these results depend. 
Elsewhere, while the results of \cite{komiyama2021} also apply specifically to Bernoulli arms, the $\Theta(1/T)$ Bayesian simple regret can also be derived for Gaussian arms, provided we are not interested in the magnitude of the constant.
\end{remark}

This paper considers a Bayes optimal algorithm that minimizes the Bayesian simple regret (Eq.~\eqref{def_reg_bayes}). 
Our main result demonstrates that the Bayes optimal algorithm does not feature an exponential frequentist simple regret, which is somewhat surprising given its optimality in the Bayesian sense.

\section{Bayes Optimal Algorithm and Loss Function}\label{sec_boa}

In this section, we formalize the Bayes optimal algorithm that minimizes the Bayesian simple regret. 

\subsection{Bayes optimal algorithm}

First, we define \replaced{of}{}the Bayes optimal algorithm, and then we represent it as a dynamic programming problem. 

\begin{definition}{\rm (Algorithms)}
An algorithm $A$ is
\begin{equation}
A = (I(1), I(2), I(3), \dots, I(T-1), I(T), J(T)),
\end{equation}
where each $I(t)$ is $\EF_{t-1}$-measurable and $J(T)$ is $\EF_T$-measurable. Let $\EA$ be the set of all algorithms.
\end{definition}

\begin{definition}{\rm (Bayes optimal algorithm)}\label{def_boa}
\added{The Bayes optimal algorithm $\ABayesOpt$ is defined as the one that minimizes the Bayesian simple regret:}
\begin{equation}
\ABayesOpt = \argmin_{A \in \EA} \Ex_{\bH}[ \SRegFreq(T, A) ].
\end{equation}
\end{definition}

\subsection{Posterior and sufficient statistics}

\added{To simplify the explanation, we introduce an improper prior that results in the posterior mean matching the empirical mean.}
\replaced{To avoid improper posteriors, we assume $T \ge K$ and $I(t) = t$ for $t \in [K]$ (i.e., the algorithm draws each arm once in the first $K$ rounds).}{To prevent improper posteriors, we assume that $T \ge K$ and that $I(t) = t$ for $t \in [K]$, meaning the algorithm draws each arm once in the first $K$ rounds.}

For $t \ge K$, the sufficient statistics $(S_i(t), N_i(t))$ for each arm $i$ are 
\begin{align}
S_i(t) &= \sum_{s \le t: I(s)=i} X_s \\
N_i(t) &= \sum_{s \le t: I(s)=i} 1, 
\end{align}
and the posterior of arm $i$ at the end of round $t\ge K$ is identified by these sufficient statistics $\bH(t) = (H_1(t), H_2(t), \dots, H_K(t))$ as
\begin{equation}
H_i(t) = \Normal\left(\frac{S_i(t)}{N_i(t)}, \frac{1}{N_i(t)}\right).
\end{equation}
The posterior on $\bmu$ is fully specified by the sufficient statistics. Henceforth, we use the posterior $\bH(t)$ and the sufficient statistics $(S_i(t), N_i(t))_{i \in [K]}$ interchangeably.
We also recognize that $\hatmu_i(t) = S_i(t)/N_i(t)$ and $(\hatmu_i(t), N_i(t))_{i \in [K]}$ represent sufficient statistics.

\subsection{\added{State evolution}}

\replaced{
The following considers the dynamic programming (Bellman equation) that defines $\ABayesOpt$, which at any moment makes a decision to minimize the simple regret marginalized by the posterior.}{The computation of the Bayes optimal algorithm can be formalized as an undiscounted episodic Markov decision process (MDP). In particular, the state at the current round $t$ is represented as a set of sufficient statistics $\bHrec(t-1) = \bH(t-1)$.}
We use $\Xrec(\cdot)$, $\Nrec(\cdot), \hatmurec_i(\cdot)$, $\Irec(\cdot)$, $\Jrec(\cdot)$, and $ \bHrec(\cdot)$ to represent the corresponding quantity in the posterior.\footnote{\added{At any given time step $t$, an algorithm can only access the sequences $I(1),X(1),I(2),X(2),\dots,I(t-1),X(t-1)$ that have been drawn from the true distribution. The Bayes optimal algorithm at round $t$ calculates the value function and the action value function using a  recursive formula that includes samples from the future time steps $t,t+1,t+2,\dots,T$. These are virtual samples, drawn from the posterior distribution.}} 
\added{
The state transition is represented as follows: when we draw arm $i$, $\Hrec_j(t) = \Hrec_j(t-1)$ for $j \ne i$. For $\Hrec_i(t)$,
\begin{align}
\tilmu_i &\sim \Normal\left(\hatmurec_i(t-1), \frac{1}{\Nrec_i(t-1)}\right)\\
\Xrec(t) &\sim \Normal(\tilmu_i, 1)\\
\hatmurec_i(t)
&= \frac{\Nrec_i(t-1) \hatmurec_i(t-1)+\Xrec(t)}{\Nrec_i(t-1)+1}\\
\Nrec_i(t) &= \Nrec_i(t-1)+1.\label{ineq_stateevol}
\end{align}
In view of the MDP, the algorithm receives no loss\footnote{\added{This is because the values $X(1),X(2)\dots,X(T)$ do not directly affect the final objective of minimizing the simple regret. We use the term ``reward'' for the values $X(1),X(2)\dots,X(T)$, whereas we use the terms ``loss'' or ``simple regret'' for representing the quantity that we would like to minimize.}} for choosing an action because this is a pure exploration setting.
After round $T$, the optimal recommendation is $\Jrec(T) = \argmax_i \hatmurec_i(T)$, and the algorithm receives a loss}
\[
L(\bHrec(T), T+1) := \Ex_{\bHrec(T)}[\Delta_{\Jrec(T)}|\Jrec(T) = \max_i \hatmurec_i(T)],
\]
where $\Delta_{J(T)} = \mu^* - \mu_{J(T)}$.
Here, the notation $T+1$ in $L(\bHrec(T), T+1)$ represents the decision after $T$ samples\replaced{ of arm $I(1),I(2),\dots,I(T)$}{}.

\subsection{\added{Value function and action value function}}

\added{The value function after round $T$ is represented as the negative loss $-L(\bHrec(T), T+1)$. To obtain the value function at $t \le T$, we consider the following recursive formula. Assume that we know the loss function at round $t+1$. Since there is no loss during the sampling process, the negative action value function (the Q-function) is represented as the expected value of the loss function\footnote{\added{Note that the loss function $L_i(\bHrec(t-1), t)$ corresponds to the expected simple regret of the Bayes optimal algorithm conditioned on the state at round $t$ being $\bHrec(t-1)$.}} after drawing an arm:
\begin{align}
\lefteqn{
L_i(\bHrec(t-1), t)
}\\
&=
\Ex_{\tilmu_i \sim \Normal(\hatmu_i(t-1), \frac{1}{\Nrec_i(t-1)}), \Xrec(t) \sim \Normal(\tilmu_i, 1)}[ L(\bHrec(t), t+1) ]\\
&:= \sqrt{\frac{\Nrec_i(t-1)}{2 \pi}} \int 
\frac{1}{\sqrt{2 \pi}} 
\int 
L(\bHrec(t), t+1)
e^{-\frac{(\Xrec(t) - \tilmu_i)^2}{2} }
d\Xrec(t)\\
&
\hspace{14em}\times e^{-\frac{\Nrec_i(t-1)(\tilmu_i - \hatmu_i(t-1))^2}{2} }
d\tilmu_i,\label{ineq_loss_onestep}
\end{align}
where the state evolution follows Eq.~\eqref{ineq_stateevol}. 
The optimal arm (action) that maximizes the action value function is $\Irec(t) = - \argmax_i L_i(\bHrec(t-1), t)$.
}

\subsection{\added{Expected Bellman improvement}}

\replaced{As the equation makes apparent, }{}The loss function $L(\bH(t-1), t)$ is recursively defined and ``virtually impossible to solve'' \citep{powell2013optimal,Russo2020}. \replaced{However, we prove that it is possible to derive a lower bound for a particular example by introducing several novel ideas.}{However, we derive a lower bound for a specific example by introducing several novel ideas.}
\replaced{First, we introduce the expected Bellman improvement (EBI) of an arm, which corresponds to the value of a sample from the arm:}{First, leveraging the fact that the only actual loss is incurred at the terminal state, we define the expected Bellman improvement (EBI) for an arm, which represents the expected value gained from sampling that arm.}
\begin{definition}{\rm (Expected Bellman improvement of arm $i$)}\label{def_ebi}
The EBI of arm $i$ is defined as:
\begin{equation}\label{ineq_ebi}
\BI_i(\bH(t-1), t) := L(\bH(t-1), t+1) - L_i(\bH(t-1), t).
\end{equation}
\end{definition}
The value $L(\bH(t-1), t+1)$ represents the loss function when a sample at round $t$ is discarded.\footnote{Note that $t+1$ is indeed correctly placed in Eq.~\eqref{ineq_ebi}. It \replaced{includes}{indicates} one round fewer than $t$.}
Therefore, the EBI indicates the decrease in the loss function due to a single sample from arm $i$. Since $L(\bH(t-1), t+1)$ does not depend on arm $i$, \replaced{the minimization of the loss function $L_i(\bH(t-1), t)$}{maximizing of the action value function $-L_i(\bH(t-1), t)$} is equivalent \replaced{as}{to} maximizing $\BI_i(\bH(t-1), t)$. We can now \added{redefine the Bayes optimal algorithm in terms of the EBI}:
\begin{thm}{\rm (Bayes optimal algorithm)}\label{thm_boa}
The Bayes optimal algorithm $\ABayesOpt$ is such that
\begin{align}
I(t) &= \argmax_{i\in[K]} \BI_i(\bH(t-1), t)\\
J(T) &= \argmax_{i\in[K]} \hatmu_i(T),
\end{align}
and its simple regret, conditioned on posterior $\bH(t-1)$ at the beginning of round $t$, is $L(\bH(t-1), t)$.
This algorithm minimizes the Bayesian simple regret:
\begin{equation}\label{alg_bayesopt}
\ABayesOpt = \argmin_{A \in \EA} \SRegBayes(T, A).
\end{equation}
\end{thm}
That is, the Bayes optimal algorithm chooses the arm with the largest EBI. Theorem~\ref{thm_boa} holds by definition of the loss function.

\subsection{One-step lookahead approximations}

Since the Bayes optimal algorithm is difficult to calculate, several approximations have been considered.
The first is a natural one-step lookahead algorithm that draws arm
\begin{equation}
I(t) = \argmax_{i} \Ex_{H(t-1)}[\mu_{J(t)} - \mu_{J(t-1)}|I(t)=i],
\end{equation}
which corresponds to the knowledge gradient \citep{Gupta1994,frazier2008}. If one considers only the ``positive'' side of the $t$-th draw, then we have the slight modification
\begin{equation}
I(t) = \argmax_{i} \Ex_{H(t-1)}[\Ind[\mu_{J(t)} - \max_{i'} \hatmu_i(t-1)]|I(t)=i],
\end{equation}
which corresponds to the expected improvement \citep{Jones1998,bull11}. \added{Here, $\Ind[\EA]$ is the indicator function; for any event $\EA$, let $\Ind[\EA] = 1$ if $\EA$ holds or $0$ otherwise.} %

Several known results demonstrate the suboptimality of KG and EI in $K$-armed BAI, which we describe in Section~\ref{sec_related}. In this paper, the convergence rate is characterized according to the Bayes optimal algorithm rather than its one-step approximations KG and EI. We show that the Bayes optimal algorithm performs very differently from the frequentist BAI, despite not making any myopic approximation. 

\section{Main Result: Suboptimality of the Bayes Optimal Algorithm}
\label{sec_main}

The following is this paper's main theorem.
\begin{thm}{\rm (Main theorem)}\label{thm_reglowinst}
There exist $K \in \Natural$, $C > 0$, and an instance $\bmutrue \in (-\infty, \infty)^K$ such that 
\begin{equation}\label{ineq_reglowinst}
\SRegFreqTrue(T, \EA^*) = \Omega(T^{-C})
\end{equation}
for a Bayes optimal algorithm $\EA^*$. 
\end{thm}
The proof of Theorem~\ref{thm_reglowinst} explicitly constructs an instance for which Eq.~\eqref{ineq_reglowinst} holds.
Notably, Theorem~\ref{thm_reglowinst} implies the following.
\begin{remark}
Letting $\bmutrue$ be an instance satisfying Theorem~\ref{thm_reglowinst}, for \replaced{a frequentist algorithm}{an algorithm} $\AFreq \in \EA$ \replaced{, such that Eq.~\eqref{ineq_freqreg} holds}{that has a frequentist rate}, 
\begin{equation}
\limsup_{T \rightarrow \infty} \frac{\SRegFreqTrue(T, \ABayesOpt)}{\SRegFreqTrue(T, \AFreq)} = \infty.
\end{equation}
\end{remark}
That is, the frequentist regret of the Bayes optimal algorithm for some instance is arbitrarily worse than \replaced{a good frequentist BAI algorithm}{an algorithm with a frequentist rate} (e.g., \replaced{the successive rejects algorithm}{SR and uniform}).\footnote{Note also that it holds that 
$
\SRegBayes(T, \ABayesOpt) / \SRegBayes(T, \AFreq) \le 1
$
by the definition of the Bayes optimal algorithm.
}

Theorem~\ref{thm_reglowinst} implies that the nature of the Bayes optimal algorithm $\ABayesOpt$ differs substantially from \added{that of the algorithms designed for a frequentist measure}. 
A frequentist prefers an algorithm that exponentially converges uniformly over all instances, but there is some instance for which $\ABayesOpt$ does not converge exponentially.
The Bayes optimal algorithm $\ABayesOpt$ optimizes the expected performance over the prior and places greater weight on cases with a small gap between the optimal and suboptimal arms; it does not explore arms that only improve the higher-order polynomial of $T^{-1}$.

\section{Proof of Theorem~\ref{thm_reglowinst}}
\label{sec_proof}

The high-level idea behind the proof of Theorem~\ref{thm_reglowinst} can be understood by considering a three-armed instance in which $\mu_1 = \mu_2 < \mu_3$. 
With a polynomially small probability in $T^{-1}$, arm $3$ (the best arm) is underestimated, and consequently, $\hatmu_1, \hatmu_2 > \hatmu_3$. When $\hatmu_3$ is very small, the Bayes optimal algorithm allocates all samples to arms $1$ and $2$ because this more effectively reduce the Bayesian simple regret, which leads to underestimation of the best arm.

Formally, let $\Gap >0$ be an arbitrary constant. Let $\bmutrue = (\mutrue_1,\mutrue_2,\mutrue_3) = (0, 0, \Gap)$. 
For ease of discussion, let $T$ be odd and let $T'$ be such that $2 T' + 1 = T$.
Furthermore, let $\CUnder > 0$ be such that
\begin{align}\label{ineq_cularge}
118 < \frac{(C_U-6)^2}{4}
\end{align}
holds.\footnote{It suffices that $C_U \ge 30$.}
Let $S_{i,n}$, $N_{i,n}$, and $\hatmu_{i,n}$ be the values of $S_i(t)$, $N_i(t)$, and $\hatmu_i(t)$ when $N_i(t) = n$.
We also define the following events.  
\begin{align}
\EX &:= \{\hatmu_{3,1} \le -\CUnder \sqrt{\log T}\}, \\ %
\EY_i &:= \left\{|\hatmu_{i,T'}| \le 
\frac{1}{T^2}
\right\},\\
\EZ_i &:= \bigcap_{n \in [T']} %
\left\{
|\hatmu_{i,n} - \hatmu_{i,T'}| 
\le
4\sqrt{\log T}
\sqrt{
\sum_{m=n}^{T'-1} 
\frac{1}{m^2}
}
\right\},\\
\EW(t) &:= \EX \cap \EY_1 \cap \EY_2 \cap \EZ_1 \cap \EZ_2 \cap \{N_1(t-1), N_2(t-1) \le T', N_3(t-1)\le 1\}.
\end{align}
Event $\EX$ indicates that the best arm is largely underestimated in the first round.
For $i=1,2$, event $\EY_i$ indicates that the empirical mean of arms $i$ with $T'$-th draw is very close to the mean, and 
event $\EZ_i$, $i=1,2$, indicates that the mean of arms $i$ does not drift drastically. 
Event $\EW(t)$ indicates that both arms $1$ and $2$ are not drawn in more than half the rounds and that arm $3$ is only drawn once. 
For any two events $\EA$ and $\EB$, let $\EA, \EB = \EA \cap \EB$.

Henceforth, all probabilities are under the frequentist model with the true parameter $\bmutrue$. Derivations of all \replaced{Lemmas}{lemmas} appear in the appendices.
\begin{lem}{\rm (Underestimation of the best arm)}\label{lem_underest} We have
\begin{equation}
\Prob[\EX]
> \frac{T^{-\frac{(\CUnder+\Gap)^2}{2}}}{2(\CUnder+\Gap)\sqrt{2\pi \log T}} =: \funder.
\end{equation}
\end{lem}
Lemma~\ref{lem_underest} directly follows from the tail bound of the normal distribution (Lemma~\ref{lem:normpdf}). 

\begin{lem}{\rm (Probability that the two means are close)}
\label{lem_close}
The following holds:
\begin{equation}
\Prob[\EY_i] \ge \frac{1}{2} \sqrt{\frac{1}{\pi T^3}}
=: \fclose.
\end{equation}
\end{lem}
By using the fact that $\hatmu_{i,T'} \sim \Normal(0, (T')^{-1})$, Lemma~\ref{lem_close} follows from the tail bound of the normal distribution. Proofs of Lemma~\ref{lem_close} and subsequent lemmas are found in the appendix.

\begin{lem}{\rm (Non-drastic drift)}
\label{lem_smldrift}
The following holds for each arm $i = 1, 2$:
\begin{equation}
\Prob[\EY_i,\EZ_i^c] \le 2T^{-7/2} =: \fnodrift.
\end{equation}
\end{lem}
We use $\poly(T^{-1})$ to describe a quantity that is polynomial in $T^{-1}$.
According to Lemmas~\ref{lem_underest}--\ref{lem_smldrift}, $\EX$, $\EY_1$, $\EY_2$, $\EZ_1$, and $\EZ_2$ occur with probability $\poly(T^{-1})$. 
The following lemmas bound the EBI under events $\EX$, $\EY_1$, $\EY_2$, $\EZ_1$, and $\EZ_2$.
\begin{lem}\label{lem_nodraw_one}
Assume that $\EW(t)$ holds.
Then, 
\begin{equation}\label{ineq_lowerupper}
\BI_3(\bHrec(t-1), t) = O\left(T^{-\frac{(C_U-6)^2}{4}}\right)
= o\left(T^{-118}\right).
\end{equation}
\end{lem}
\begin{lem}\label{lem_nodraw_two}
Assume that $\EW(t)$ holds for a $t<T$.
Then, 
\begin{equation}
\BI_1(\bHrec(t-1), t), \BI_2(\bHrec(t-1), t) = \Omega( T^{-118} ).
\end{equation}
Moreover, if $N_j(t) = T'$, $j=1,2$, then\footnote{This implies $N_k(t) < T'$ for $k \in \{1,2\}\setminus\{j\}$.} 
\begin{equation}\label{ineq_nomorethanhalf}
\BI_j(\bHrec(t-1), t) < \max(\BI_1(\bHrec(t-1), t), \BI_2(\bHrec(t-1), t)).
\end{equation}
\end{lem}
Lemma~\ref{lem_nodraw_one} provides an upper bound for $\BI_3(t)$ under $\EW(t)$. Lemma~\ref{lem_nodraw_two} provides lower bounds for $\BI_1(t)$ and $\BI_2(t)$ under $\EW(t)$. 

Given these lemmas, we can provide the main theorem.
\begin{proof}[Proof of Theorem \ref{thm_reglowinst}]
It is trivial to show that  $\{\EX,\EY_1,\EY_2,\EZ_1,\EZ_2\}$ implies $\EW(4)$.
We first show that for $t \in \{4,5,\dots,T\}$, $\EW(t)$ implies $\EW(t+1)$. 
Lemmas~\ref{lem_nodraw_one} and~\ref{lem_nodraw_two} imply that $\BI_3(\bH(t-1), t) < \BI_1(\bH(t-1), t)$ and thus arm $3$ is not drawn. Moreover,  
Lemma~\ref{lem_nodraw_two} states that arm $i \in \{1,2\}$ is not drawn if $N_i(T)=T'$. Event $\EW(t+1)$ follows from these facts.

Event $\EW(T+1)$ implies $N_1(T)=N_2(T)=T'$ and $N_3(T)=1$. We have that
\begin{align}
\hatmu_1(T) &= \hatmu_{1,T'}\\
&\ge -4\sqrt{\log T}
\sqrt{
\sum_{m=n}^\infty 
\frac{1}{m^2}
}-\frac{1}{T^2} \text{\ \ \ \ (by $\EY_1, \EZ_1$)} \\
&\ge -\sqrt{\frac{8 \pi^2\log T}{3}} - \frac{1}{T^2} \ge -6\sqrt{\log T} \\
& > -\CUnder \sqrt{\log T} \\
&\text{\ \ \ \ (by Eq.~\eqref{ineq_cularge})}\\
&\ge \hatmu_{3,1} = \hatmu_3(T),
\end{align}
and thus $J(T) \ne 3$.

In summary, 
event $\{\EX,\EY_1,\EY_2,\EZ_1,\EZ_2\}$ implies $\EW(4)$. Event $\EW(4)$ implies $\EW(5),\EW(6),\dots,\EW(T+1)$. Moreover, 
event $\EW(T+1)$ implies $J(T) \ne 3$.
Therefore, we have
\begin{align}
\SRegFreqTrue(T) 
&\ge \Prob[J(T) \ne 3] \\
&\ge \Prob[\EX,\EY_1,\EY_2,\EZ_1,\EZ_2] \Gap \\
&\ge \funder (\fclose - \fnodrift)^2 \Gap, \text{\ \ \ \ (by Lemmas~\ref{lem_underest}--\ref{lem_smldrift})}
\end{align}
which is $\poly(T^{-1})$ if we view $\Gap$ as a constant.
\end{proof}

\section{Proof Sketches of Lemmas~\ref{lem_nodraw_one} and~\ref{lem_nodraw_two}}
\label{sec_proofsketch}

\added{(Note: This section is moved from Appendix to Main Paper)}

\added{This section describes the main ideas of the proofs of Lemmas~\ref{lem_nodraw_one} and~\ref{lem_nodraw_two}.}
Lemmas~\ref{lem_nodraw_one} and~\ref{lem_nodraw_two} are very different from the previous lemmas.
Whereas Lemmas~\ref{lem_underest}--\ref{lem_smldrift} deal with probabilities under the true parameter $\bmutrue$, Lemmas~\ref{lem_nodraw_one} and~\ref{lem_nodraw_two} are about how the Bayes optimal algorithm values the sample from each arm $i$ at round $t$, which is represented by the recursive formula \added{of Eq.~\eqref{ineq_loss_onestep}}.
In particular, the latter lemmas bound the EBI $\BI_i(\bH(t-1), t) = L(\bH(t-1), t+1) - L_i(\bH(t-1), t)$ given the true posterior $\bH(t-1)$. 
We first describe the evolution of $\bHrec(t-1)$ in Section~\ref{sec_posterior_evol}.
Derivation of Lemmas~\ref{lem_nodraw_one} and~\ref{lem_nodraw_two} are sketched in Sections~\ref{subsec_sketch_one} and~\ref{subsec_sketch_two}, respectively. 

\subsection{Evolution of the posterior in the loss function}
\label{sec_posterior_evol}

This section describes the evolution of $\bHrec(s)$ for $s \ge t-1$. At the beginning of round $t$, the Bayes optimal algorithm only observes $I(1),X(1),I(2),X(3),\dots,I(t-1),X(t-1)$, which were drawn from the true distribution:\footnote{Note that we consider the frequentist setting in which the (unknown) true parameters exist.}
\begin{equation}
X(t') \sim \Normal(\mu_{I(t')}, 1).
\end{equation}

The value $L_i(\bH(t-1), t)$ involves future rewards 
$\Irec(t), \Xrec(t),\Irec(t+1),\Xrec(t+1),\dots,\Irec(T),\Xrec(T)$, which are random variables \replaced{computed in Eqs.~\eqref{ineq_recursion} and \eqref{ineq_recursion_xdraw}. These future rewards}{that} are drawn not from the true distribution\footnote{The Bayes optimal algorithm does not know the true parameter $\bmu$ and believes in the posterior $\bH(t-1)$.} but from the posterior.
Namely, for $s\ge t-1$ and $i = I(s+1)$,
\begin{align}
\tilmu_i(s) &\sim \Normal\left(\added{\hatmurec_i(s)}, \frac{1}{\Nrec_i(s)}\right) \\
\Xrec(s+1) &\sim \Normal(\tilmu_i(s), 1)\\
\hatmu_i(s+1) &= \frac{\Nrec_i(s)\hatmurec_i(s)+\Xrec(s+1)}{\Nrec_i(s)+1},
\end{align}
or equivalently, 
\begin{equation}\label{ineq_bhrecevol}
\hatmurec_i(s+1) - \hatmurec_i(s)
\sim 
\Normal\left(\hatmurec_i(s), \frac{1}{\Nrec_i(s)(\Nrec_i(s)+1)}\right)
\end{equation}
and the posterior is updated accordingly. 
In other words, drawing arm $i$ with its posterior $\Normal(\hatmurec_i(s), 1/\Nrec_i(s))$ reduces its variance from $1/\Nrec_i(s)$ to $1/(\Nrec_i(s)+1)$.

\subsection{Sketch of the proof of Lemma~\ref{lem_nodraw_one}}
\label{subsec_sketch_one}

Lemma~\ref{lem_nodraw_one} provides an upper bound for the EBI of arm $3$ when the posterior probability of the arm being best is small.
In the proof of Lemma~\ref{lem_nodraw_one}, we derive Lemmas~\ref{lem_oneoptpost} and~\ref{lem_upper}. Lemma~\ref{lem_oneoptpost} bounds 
\[
\Ex_{\bHrec(t-1)}[ \Ind[\ist(\bmu)=3]\SRegFreq(T) ],
\]
which is a portion of the Bayesian simple regret when arm $3$ is optimal.
By using this, Lemma~\ref{lem_upper} bounds the EBI.

The greatest challenge in bounding the EBI $\BI_i(\bHrec(t-1), t) = L(\bHrec(t-1), t+1) - L_i(\bHrec(t-1), t)$ lies in its adaptivity. A slight difference in the posterior $\bHrec(s)$ can change the selection sequence $\Irec(s),\Irec(s+1),\dots,\Irec(T)$, and thus it is generally difficult to compute the EBI exactly. 
The high-level idea of Lemma~\ref{lem_upper} is as follows. When an arm is very unlikely to be the best arm (i.e., under event $\EX$), the \added{next} sample for that arm does not significantly improve the simple regret.
Notice that the difference between $L(\bHrec(t-1), t+1)$ and $L_i(\bHrec(t-1), t)$ is derived from the draw for arm $i$ at round $t$. To bound the EBI, we introduce a \textit{mock-sample} algorithm, along with its loss $\Lmock(\bHrec(t-1), t)$, which utilizes the external randomizer $\Urec(t)$ such that
\begin{align}
\Urec(t)&, \Irec(t+1), \Xrec(t+1), \Irec(t+2), \Xrec(t+2), \dots, \Irec(T), \Xrec(T)\\
&\text{\ \ of $\Lmock(\bHrec(t-1), t)$, and }\\
\Xrec(t)&, \Irec(t+1), \Xrec(t+1), \Irec(t+2), \Xrec(t+2), \dots, \Irec(T), \Xrec(T)\\
&\text{\ \ of $L_i(\bHrec(t-1), t)$}
\end{align}
have exactly the same distribution. We have that
\begin{align}
\BI_i(\bHrec(t-1), t) 
&:= L(\bHrec(t-1), t+1) - L_i(\bHrec(t-1), t)\\
&\le \Lmock(\bHrec(t-1), t) - L_i(\bHrec(t-1), t)\\
&\text{\ \ \ (by optimality of $L(\bHrec(t-1), t+1)$)}\\
&\le \Ex_{\bHrec(t-1)}[ \Ind[\ist(\bmu)=3]\SRegFreq(T) ].
\end{align}

\subsection{Sketch of the proof of Lemma~\ref{lem_nodraw_two}}
\label{subsec_sketch_two}

Lemma~\ref{lem_nodraw_two} provides a lower bound for 
the EBI $\BI_i(\bHrec(t-1), t) = L(\bHrec(t-1), t+1) - L_i(\bH(t-1), t)$ for $i = 1,2$.
Note that we define the events $\EY_1$, $\EY_2$, $\EZ_1$, and $\EZ_2$ such that $\hDelta(t) := \hatmu_1(t)-\hatmu_2(t)$ is small. 
If $\hatmu_1(t)-\hatmu_2(t)$ is small enough and arm $3$ is sufficiently worse than the other two arms, we can expect that drawing arm $1$ or $2$ reduces the loss function by $\poly(T^{-1})$. Again, evaluating $L(\bHrec(t-1), t+1) - L_i(\bHrec(t-1), t)$ exactly is very challenging in general, as the draws $\Irec(t),\Irec(t+1),\Irec(t+2),\dots,\Irec(T)$ depend on the \added{recursive computation of the future} rewards. We solve this problem in steps as follows: 
\begin{enumerate}[(A)]
\item  First, when the probability of arm $3$ being best is very small (i.e., under event $\EX$), we can show that 
\[
|L(\bHrec(t-1), t+1) - L(\Hexrec{3}(t-1), t+1)|
\]
is also small, where $\Hexrec{3}(t-1)$ is the posterior in the two-armed instance with arm $3$ removed. In other words, we have reduced a three-armed instance to a two-armed instance (Section~\ref{subsec_reduction}).
\item Second, in the two-armed instance, the Bayes optimal algorithm always draws the arm with the smaller number of draws. Therefore, we can explicitly obtain the distribution of $\hDeltarec(T-1)$ given $\hDelta(t-1)$. In the two-armed problem, we show that $\EW(t)$ implies a sufficiently small $\hDelta(t-1)$, which yields a sufficiently small $\hDeltarec(T-1)$ with probability $\Omega(\poly(T^{-1}))$, providing an $\Omega(\poly(T^{-1}))$ value for $L(\Hexrec{3}(T-1), T+1) - L_i(\Hexrec{3}(T-1), T)$. These results yield an $\Omega(\poly(T^{-1}))$ lower bound for the EBI $\BI_i(\Hexrec{3}(t-1), t)$ in the two-armed instance (Section~\ref{subsec_twobound}). 
\item Finally, by using (A) and (B), we derive Lemma~\ref{lem_nodraw_two} (Section~\ref{subsec_lower_final}).
\end{enumerate}

As a byproduct of the analysis, we have obtained a  complete characterization of the Bayes optimal algorithm in the case of two-armed best arm identification.
\begin{remark}{\rm (Characterization of the Bayes optimal algorithm for $K=2$)}\label{rem_twoarm}
Lemma~\ref{lem_drawnumdiff} in the appendices implies that $L_1(\Hexrec{3}(t-1), t) = L_2(\Hexrec{3}(t-1), t)$ if $N_1(t), N_2(t) < T'$, and that $L_1(\Hexrec{3}(t-1), t) > L_2(\Hexrec{3}(t-1), t)$ if $N_1(t) \ge T', N_2(t) < T'$. A Bayes optimal algorithm eventually draws both of the arms $T'$ times.
\end{remark}

\section{Simulation\label{sec_sim_brief}}

We conducted a computer simulation to observe the polynomial rate of the simple regret.\footnote{\added{The code that replicates the results is available at \url{https://github.com/jkomiyama/bayesoptimalalg/}.}}

\subsection{\added{Goal}}

\added{The goal of this simulation was to support the theoretical results.
Since the Bayes optimal algorithm is almost intractable, we used an approximate Bayes optimal (ABO) algorithm that is computationally efficient.} 
\added{
We show that the ABO algorithm has simple regret polynomial in $T^{-1}$.
}

\subsection{\added{Approximate Bayes optimal algorithm}}

\added{
A Bayes optimal algorithm draws an arm that maximizes the EBI (Definition~\ref{def_ebi}):
\[
\BI_i(\bH(t-1), t) := L(\bH(t-1), t+1) - L_i(\bH(t-1), t).
\]
Here, $L(\bH(t-1), t+1)$ is the Bayesian simple regret when the algorithm skips one round and explores optimally for the remaining rounds, whereas $L_i(\bH(t-1), t)$ is the Bayesian simple regret when the algorithm draws arm $i$ and explores optimally thereafter. Essentially, the difference of these two values indicates the value of the sample on arm $i$ to the loss function. Since the loss function is difficult to compute, we use the approximate indices $L_{\bH(t-1)}^{\mathrm{ABO}}$ and $L_{\bH_i^A(t)}^{\mathrm{ABO}}
$ to represent the value of drawing arm $i$, where
}
\replaced{
Although this does not compute the recursive equations, it simulates the behavior of the Bayes optimal algorithm where the value of drawing a largely underestimated arm is very low.
This section aims to demonstrate that, unlike frequentist algorithms, the Bayes optimal algorithm features polynomial regret. In particular, conditioned on a bad underestimation of level $\poly(T^{-1})$ in the initial stage, the simple regret of the Bayes optimal algorithm is $\poly(T^{-1})$. For the sake of comparison, we also test the successive rejects algorithm, which has an exponentially small regret for any instance.
}{}
\begin{align}
L_{\bH(t-1)}^{\mathrm{ABO}} = \sum_{i: \hatmu_i(t-1) < \hatmu_{\hat{i}^*}(t-1)} \Ex_{\bH(t-1)}\left[ \left(\mu_i - \mu_{\hat{i}^*}\right)_+ \right]
\end{align}
and
\begin{align}
\bH_i^A(t) &= \bigl((\hatmu_1(t-1),N_1(t-1)),(\hatmu_2(t-1),N_2(t-1)),\\
&\ \ \ \ \ \ \ \ \ \dots,
 (\hatmu_i(t), N_i(t-1)+1),
 \dots,(\hatmu_K(t-1),N_K(t-1))\bigr),
\end{align}
$(x)_+ = \max(x, 0)$, and $\hat{i}^* = \argmax_i \hatmu_i(t-1)$.
\replaced{More concretely, $\bH_i^A(t)$ features smaller variance on arm $i$ than $\bH(t-1)$.
}{Here, $\bH_i^A(t)$ is an identical copy of $\bH(t-1)$ except for a smaller variance on arm $i$ that corresponds to an additional sample.}
ABO draws the arm that maximizes the decrease in $L_{\bH(t-1)}^{\mathrm{ABO}}$. Specifically,
\begin{equation}
I(t) = \argmax_{i}
\left(
L_{\bH(t-1)}^{\mathrm{ABO}} -
L_{\bH_i^A(t)}^{\mathrm{ABO}}
\right),
\end{equation}
which represents a one-step lookahead.

\subsection{\added{Experimental settings}}

We considered a three-armed bandit problem with $\bmutrue = (0, 0, 0.5)$. 
\added{Instead of running a standard bandit simulation, we modified the first reward of arm $3$, drawing it from a truncated normal distribution $TN_{[-\infty, - \sqrt{4 \log T}]}(\mutrue_3, 1)$ that takes a value on $[-\infty, - \sqrt{4 \log T}]$. This approximately corresponds to the worst $T^{-2}$-th percentile of the original distribution.
The rest of the simulation was exactly the same as  standard best arm identification. Namely, a reward from arm $i$ was drawn from $\Normal(\mutrue_i, 1)$.
}
As demonstrated in Section~\ref{sec_proof}, the polynomial regret of ABO results from the very low posterior of arm $3$ (i.e., event~$\EX$)\added{, and our setting is designed to reproduce such cases more frequently}. 
\replaced{To reproduce such a case, we draw the first reward of arm $3$ from $\Normal(- \sqrt{4 \log T} + \mutrue_3, 1)$, which approximately corresponds to the worst \replaced{$T^{-4}$}{$T^{-2}$}-th percentile.}{} 
\added{We tested several different values of the horizon $T$.} For each $T$, we simulated \replaced{$1,000$}{$100$} runs and found the standard two-sigma confidence region. 

\subsection{\added{Experimental results}}

Figure~\ref{fig:regret2} compares the regret in SR and ABO. Unlike the successive rejects algorithm, the \replaced{Bayesian simple regret}{regret of ABO} remains large, even for large $T$, suggesting that the simple regret of the Bayes optimal algorithm is \replaced{$\Omega(T^{-4}) \times \Omega(1) = \Omega(T^{-4})$}{polynomial in $T^{-1}$ unlike SR that has a regret exponentially small in $T$}. 

Moreover, Figure~\ref{fig:regret3} shows the fraction of the runs in which arm $3$ receives a very small (5\% of $T$) number of samples, demonstrating that the underestimation of arm $3$ \added{consistently occurs for all values of $T$. This supports our finding that the initial underestimation of the best arm persists with a probability polynomial in $T^{-1}$.}
\begin{figure}[t!]
    \centering
    \begin{minipage}[t]{0.48\textwidth}
         \centering
         \includegraphics[width=\textwidth]{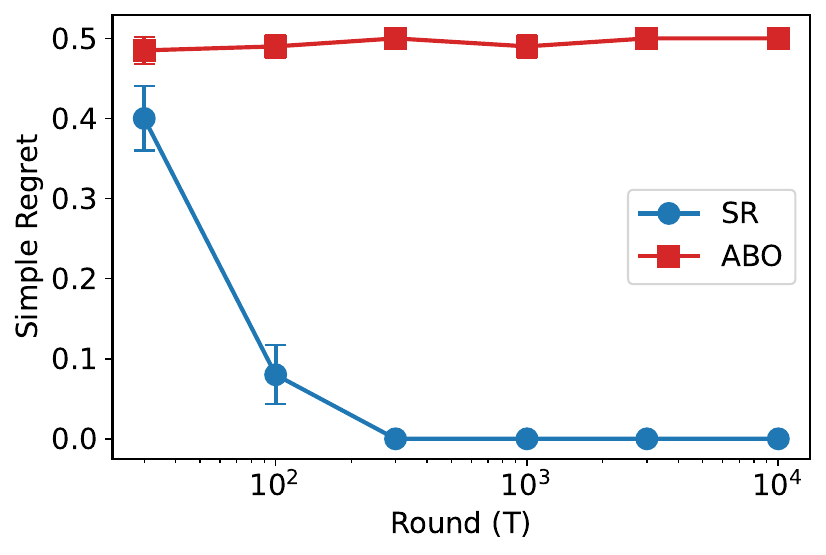}
         \caption{Simple regret of SR and ABO.}
         \label{fig:regret2}
    \end{minipage}
    \hspace{0.02\textwidth}
    \begin{minipage}[t]{0.48\textwidth}
         \centering
         \includegraphics[width=\textwidth]{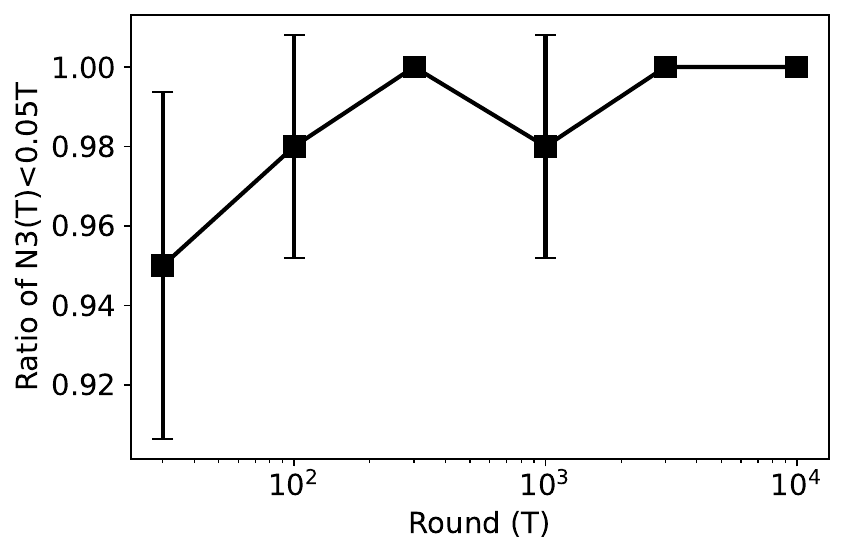}
         \caption{The fraction of runs for which $N_3(T) < 0.05 T$ in ABO.}
         \label{fig:regret3}
    \end{minipage}
\end{figure}

\section{Related Work \label{sec_related}}

This section reviews related studies concerning Bayesian algorithms in the context of sequential decision-making.
The multi-armed bandit problem \citep{Robbins1952} involves multiple treatment arms and decisions made using samples obtained sequentially via experiments. The aim is to maximize the sum of the rewards, which boils down to cumulative regret minimization (CRM). Although this problem has attracted considerable attention from the machine learning community, this study has focused on another established branch of bandit problems, BAI, in which the goal is to identify the treatment arm with the largest mean. Finding the best treatment essentially concerns simple regret minimization (SRM). 
The different goals of these two approaches results in significant differences between the approaches of CRM and SRM algorithms to balancing exploration and exploitation.

\noindent\textbf{Best arm identification (SRM):} 
Although the term ``best arm identification'' was coined in the early 2010s \citep{Audibert10,Bubeck2011}, similar ideas had already attracted substantial attention in various fields \citep{Paulson1964,Maron1997,EvanDar2006}.
\cite{Audibert10} proposed the successive rejects algorithm, which demonstrates frequentist simple regret with the order of $\exp(- \Gamma(\bmu) T)$ for some $\Gamma(\bmu)$. Elsewhere, \cite{Carpentier2016} showed an instance $\bmu$ for which this $\Gamma(\bmu)$ is optimal up to a constant factor.
Sometimes, the probability of error $\Prob[J(T) \ne \ist]$, which has the same exponential rate as the frequentist simple regret (\citealt{Audibert10}, Section 2), is used to measure a BAI algorithm.

\added{The aforementioned works focus on fixed-budget identification, in which the horizon $T$ is fixed. There is also a significant body of literature devoted to fixed-confidence identification. In this scenario, the forecaster is given a confidence level $\delta$ and aims to stop the sampling as soon as they are confident in the identification with 
a probability of $1-\delta$. For fixed-confidence identification, several algorithms are known, such as successive elimination \citep{EvanDar2006}, track and stop \citep{Garivier2016}, and top-two algorithms \citep{Russo2020,Xuedong2020,DBLP:conf/nips/JourdanDBHK22,DBLP:conf/colt/YouQWY23}. The designs of fixed-confidence and fixed-budget algorithms are markedly different \citep{Carpentier2016,komiyama2022fbbai,DBLP:conf/colt/Degenne23}.}

\noindent\textbf{Ranking and selection:} 
A particularly interesting strand of research concerns the RS problem \citep{Hong2021review}. 
Although RS and BAI are both used to find the best arm, they differ considerably in that RS considers a static allocation of samples based on knowledge of the true model parameter. 
The optimal computing budget allocation approach initiated by \cite{chen2000} is among the most popular ways to determine a static allocation. Given knowledge of the model parameters, this method solves the optimization problem
\begin{equation}\label{staticopt}
\max_{\sum_j N_j = T} \left(
\hatmu_{\ist}(N_{\ist}) > \max_{j \ne \ist} \hatmu_j(N_j)
\right).
\end{equation}
Elsewhere, \cite{glynn2004large} observed an asymptotically optimal allocation for Eq.~\eqref{staticopt} based on the large deviation theory. 
Additional adaptive algorithms have been developed based on Bayesian ideas. For example, the seminal work by \cite{chickinoue2001} proposed a two-stage method called the expected value of information algorithm, in which the information from the first stage is used as a prior to optimize the allocation of the second stage.
The KG algorithm \citep{Gupta1994} also adopts a Bayesian approach. %
However, the analytical properies of KG are little understood.
The results of \cite{RyzhovPF12} show the discounted version of KG most frequently drawing the best arm. Elsewhere, \cite{wang2018kg} have further characterized KG without giving a rate for the simple regret.
Another well-known one-step lookahead algorithm is EI \citep{Jones1998,bull11}. \cite{Ryzhov2016} showed that EI draws suboptimal arms $\Theta(\log T)$ times, implying the suboptimality of EI in terms of the frequentist simple regret (or probability of error). \cite{Qin2017} and \cite{chen_ryzhov_2019} proposed modified versions of EI with an exponential rate of convergence.

\noindent\textbf{Bayesian algorithms for CRM:}
Thompson sampling \citep{Thompson1933} is among the oldest of the heuristics and is known to be asymptotically optimal in terms of the frequentist CRM \citep{agrawal2012,kaufmann2012,komiyama15}. One of the seminal results regarding Bayesian CRM is the Gittins index theorem \citep{Gittins89,weber1992gittins}, which states that minimizing the discounted Bayesian regret can be achieved by computing the Gittins indices of individual arms.\footnote{Note that the Gittins index is no longer optimal in the context of undiscounted regret \citep{berry1984}; see also Section 1.3.3 in \cite{kaufmannthesis}.}
The Bayes optimal algorithms for CRM and SRM differ considerably. In the former setting, the algorithm obtains a reward at the end of each round. In the latter setting, the corresponding reward is delayed until round $T$, which makes our problem particularly challenging.

\section{Discussion}
\label{sec_disc}
This paper has considered the BAI problem. We have demonstrated that the Bayes optimal algorithm, which is optimized for the expected performance over the prior, does not have \replaced{the exponential simple regret that characterizes frequentist bandit algorithms}{a frequentist rate of simple regret}. In some distributions, the Bayes optimal algorithm does not perform well, even when the distributions are covered by the prior, simply because such distributions \added{carry limited weight with respect to} the prior expectation. 
A high-level implication is that if an approximate Bayesian algorithm \added{exhibits} a uniform exponential convergence, this results not from considering a lookahead but from some idea that lends it the robustness preferred by frequentists.
A challenge in analyzing a sequential decision-making algorithm is its flexibility. The decision $I(t)$ at round $t$ depends on the results of the Bellman equation, which is difficult to compute exactly. Accordingly, we have introduced a quantity called the EBI, which represents the improvement of the Bayesian simple regret from a single sample. 
We have introduced several novel ideas for analyzing the EBI and obtained a complete characterization of the Bayes optimal algorithm for two-armed BAI as a byproduct.

It is worth remarking on the unusual instance from which we derived Theorem~\ref{thm_reglowinst}. 
We have provided an instance in which $\mu_1 = \mu_2$. In this case, an algorithm cannot identify $\mu_1$ from $\mu_2$, and the value of drawing arm $2$ remains significant even when it is close to optimal. Such situations do not occur when $\mu_1$ and $\mu_2$ are distinct. 
This is a limitation of our results and highlights the following open problem.
\begin{conjecture}{\rm (Open problem: Exponential convergence of the Bayes optimal algorithm for almost all parameters)}
Under some regularity conditions over the prior, the Bayes optimal algorithm has a regret that falls exponentially with $T$ for almost all the instances.
\end{conjecture}    

\bibliographystyle{unsrtnat}
\bibliography{references}
\appendix

\section{General Lemma}

\begin{lem}[Normal tail bound; \citealt{feller-vol-1}]\label{lem:normpdf}
Let $\phi(x) := e^{-x^2/2}/\sqrt{2 \pi}$ be the probability density function (pdf) of a standard normal random variable. 
Let $\Phi^c(x) = \int_{x}^\infty \phi(x') dx'$.
Then, for any $x>0$, the following inequality holds:
\begin{equation}
\left(\frac{1}{x} - \frac{1}{x^3}\right) \frac{e^{-x^2/2}}{\sqrt{2 \pi}}
\le
\Phi^c(x)
\le 
\frac{1}{x} \frac{e^{-x^2/2}}{\sqrt{2 \pi}}.
\end{equation}
\end{lem}

\section{Proof of Lemma~\ref{lem_close}}

Since $\hatmu_{1,T'}$ is a zero-mean normal random variable with variance ${T'}^{-1} \le 2/T$, we have
\begin{align}
\Prob\left[
|\hatmu_{i,T'}| \le 
\frac{1}{T^2}
\right]
&=
2 \sqrt{\frac{T}{4\pi}}\int_0^{T^{-2}} \exp(-Tx^2/4) dx \\
&\ge \sqrt{\frac{1}{\pi T^3}}\exp\left(-1/(4T)\right)
\ge \sqrt{\frac{1}{\pi T^3}} \left(1 - \frac{1}{2T^3}\right) \ge \frac{1}{2} \sqrt{\frac{1}{\pi T^3}}.
\end{align}

\section{Proof of Lemma~\ref{lem_smldrift}}

We have that,
\begin{align}
\lefteqn{
\Prob\left[
|\hatmu_{i,T'}| \le \frac{1}{T^2}
,
\bigcup_{n \in [T']} 
\left\{
|\hatmu_{i,n} - \hatmu_{i,T'}| 
>
4\sqrt{\log T}
\sqrt{
\sum_{m=n}^{T'-1}
\frac{1}{m^2}
}
\right\}
\right]
}
\\
&\added{=}
\Prob\left[
\bigcup_{n \in [T']} 
\left\{
|\hatmu_{i,T'}| \le \frac{1}{T^2}
,
|\hatmu_{i,n} - \hatmu_{i,T'}| 
>
4\sqrt{\log T}
\sqrt{
\sum_{m=n}^{T'-1}
\frac{1}{m^2}
}
\right\}
\right]
\\
&\le
\Prob\left[
\bigcup_{n \in [T']} 
\left\{
|\hatmu_{i,T'}| \le \frac{1}{T^2}
,
|\hatmu_{i,n} - \hatmu_{i,T'}| 
>
4\sqrt{\log T}
\sqrt{
\int_{m=n}^{T'}
\frac{1}{m^2} dm
}
\right\}
\right]
\\
&=
\Prob\left[
\bigcup_{n \in [T']} 
\left\{
|\hatmu_{i,T'}| \le \frac{1}{T^2}
,
|\hatmu_{i,n} - \hatmu_{i,T'}| 
>
4\sqrt{\log T}
\sqrt{
\frac{1}{n} - \frac{1}{T'}
}
\right\}
\right]
\\
&\le
\Prob\left[
\bigcup_{n \in [T']} 
\left\{
|\hatmu_{i,T'}| \le \frac{1}{T^2}
,
|\hatmu_{i,n} - \hatmu_{i,T'}| 
>
4\sqrt{\log T}
\frac{
\sqrt{
T'-n
}
}{T'}
\right\}
\right].
\end{align}
Letting \added{$\hatmu_{i,n:T'}$} be the empirical average of the last $T'-n$ samples, we have
\begin{align}
\lefteqn{
\Prob\left[
|\hatmu_{i,T'}| \le \frac{1}{T^2}
,
|\hatmu_{i,n} - \hatmu_{i,T'}| 
>
4\sqrt{\log T}
\frac{
\sqrt{
T'-n
}
}{T'}
\right]
}
\\
&\le 
\Prob\left[
|\added{\hatmu_{i,n:T'}}| 
>
4\sqrt{\log T}
\frac{1}{ \sqrt{T'-n} }
-
\frac{1}{T^2}
\right]
\\
&\le 
\Prob\left[
|\added{\hatmu_{i,n:T'}}| 
>
3\sqrt{\log T}
\frac{1}{ \sqrt{T'-n} }
\right]\\
&\le 
\Prob\left[
|\added{\hatmu_{i,n:T'}}| 
>
3\sqrt{\log T}
\frac{1}{ \sqrt{T'-n} }
\right]\\
&\le 
\exp\left(-\frac{(3\sqrt{\log T})^2}{2}\right) 
\\&
\text{\ \ \ (by Lemma~\ref{lem:normpdf})}
\\&= T^{-9/2}
\end{align}
and thus
\begin{align}
\lefteqn{
\Prob\left[
\bigcup_{n \in [T']} 
\left\{
|\hatmu_{i,T'}| \le \frac{1}{T^2}
,
|\hatmu_{i,n} - \hatmu_{i,T'}| 
>
3\sqrt{\log T}
\frac{
\sqrt{
T'-n
}
}{T'}
\right\}
\right]
}\\
&\le 
\sum_{n \in [T']} 
\Prob\left[
|\hatmu_{i,T'}| \le \frac{1}{T^2}
,
|\hatmu_{i,n} - \hatmu_{i,T'}| 
>
3\sqrt{\log T}
\frac{
\sqrt{
T'-n
}
}{T'}
\right]\\
&\le T \times T^{-9/2} = T^{-7/2}.
\end{align}

\section{Proof of Lemma~\ref{lem_nodraw_one}}
\label{sec_upper}

In this section, we obtain an upper bound for $\BI_i = L(\bHrec(t-1), t+1) - L_i(\bHrec(t-1),t)$. 
The gap between $L(\bHrec(t-1), t+1)$ and $L_i(\bHrec(t-1), t)$ is derived from the draw of arm $i$ at round $t$. However, bounding the difference is challenging because a slight change in the posterior can drastically affect the subsequent choice of arms.
We introduce a mock-sample loss that simulates the Bayes optimal algorithm with a slightly modified posterior so that the difference between $\bHrec(t-1)$ and $\bHrec(t)$ is absorbed. 
In the proof, we derive Lemmas~\ref{lem_oneoptpost} and~\ref{lem_upper}. Lemma~\ref{lem_nodraw_one} is a straightforward corollary of these two lemmas.

\begin{lem}\label{lem_oneoptpost}
Assume that event $\EW(t)$ holds. Then,
\begin{equation}
\SRegBayesi{\bHrec(t-1)}{3}
:=
\Ex_{\bHrec(t-1)}[ \Ind[\ist(\bmu)=3]\SRegFreq(T) ]
\le
\gthree
\end{equation}
where 
$\gthree = \gthree(T, \CUnder)$ is $O\left(T^{-\frac{(C_U-6)^2}{4}}\right)$.
\end{lem}
\begin{proof}[Proof of Lemma~\ref{lem_oneoptpost}]
We first bound $\hatmu_1(t)$ and $\hatmu_2(t)$ from below:
\begin{align}
\hatmu_i(\added{t-1}) 
&= \hatmu_{i,N_i(\added{t-1})}\\
&\ge - 4\sqrt{\log T}
\sqrt{
\sum_{m=1}^{\infty} 
\frac{1}{m^2}
}
-\frac{1}{T^2}
\\
&\text{\ \ \ (by $\EY_i\EZ_i$)}\\
&\ge -\sqrt{\frac{8 \pi^2\log T}{3}} - \frac{1}{T^2}\\ &\ge -6\sqrt{\log T}.\label{ineq_muone}
\end{align}

Let $C_F = 6$. Let $x\sim \Normal(-\CUnder\sqrt{\log T}, 1)$, $y \sim \Normal(-C_F\sqrt{\log T}, 1)$ be two normal random variables.
\added{Note that $\EW(t)$ implies that $\mu_3$ follows a normal distribution with a mean smaller than $\CUnder$, and Eq.~\eqref{ineq_muone} implies that $\mu_1$ follows a normal distribution with a mean larger than $C_F$.
By using these statements, }we have
\begin{align}
\lefteqn{
\Ex_{\bHrec(t-1)}[ \Ind[\ist(\bmu)=3]\SRegFreq(T) ]
}\\
&\le 
\Ex_{\bHrec(t-1)}\left[
\Ind[(\mu_3 - \mu_1)_+]
\right] 
\\
&\le
\Ex\left[
\Ind[(x - y)_+]
\right] \text{\ \ \ \added{(by $C_F < \CUnder$)}}
\\
&\le
\frac{1}{\sqrt{4\pi}}\int_{\CUnder - C_F}^\infty \exp(-z^2/4) dz
\\&
\text{\ \ \ (by $z := x-y \sim \Normal(-(C_F-\CUnder), 2)$\ )} %
\\&=: \gthree,
\end{align}
which is $O\left(T^{-\frac{(C_U-C_F)^2}{4}}\right) = O\left(T^{-\frac{(C_U-6)^2}{4}}\right)$. 
\end{proof}

\begin{lem}{\rm (Upper bound of the EBI)}\label{lem_upper}
For any $t \in \Natural$, the following inequality holds:
\begin{equation}
\BI_i(\bHrec(t-1), t) \le \SRegBayesi{\bHrec(t-1)}{i}.
\end{equation}
\end{lem}
\begin{proof}[Proof of Lemma~\ref{lem_upper}]
By definition,
\[
\BI_i(\bHrec(t-1), t) := L(\bHrec(t-1), t+1) - L_i(\bHrec(t-1), t).
\]

We first introduce the notion of mock-sample loss.
\begin{definition}{\rm (Mock-sample loss)}
Let 
\begin{align}
\mu_i(t-1) &\sim \Normal\left(\hatmurec_i(t-1), \frac{1}{\Nrec_i(t-1)}\right)\\
\Urec(t) &\sim \Normal(\mu_i(t-1), 1).
\end{align}
That is, $\Urec(t)$ is a simulated sample that is obtained from the external randomizer. Moreover, let 
\[
(\SPrec_j, \NPrec_j) =
\left\{
\begin{array}{ll} 
(\Srec_i(t-1)+\Urec(t), \Nrec_i(t-1)+1)& \text{($j = i$)} \\
(\Srec_j(t-1), \Nrec_j(t-1)) & \text{($j\ne i$)}
\end{array}
\right.
\]
be the updated posterior.
We define $\Lmock(\bHrec(t-1), t)$ as the loss of the recursive equation that determines the adaptive sequence $\Irec(t+1),\Irec(t+2),\Irec(t+3)$ as if $(\SPrec_j, \NPrec_j)_{j\in[K]}$ were $\bHrec(t)$ and that recommends $\Jrec(T) = \argmax_{j \in [K]\setminus\{i\}} \hatmurec_j(T)$, which is the best recommendation that excludes arm $i$. 
\end{definition}

Then, the following inequalities hold:
\begin{equation}
L(\bHrec(t-1), t+1) \le \Lmock(\bHrec(t-1), t) \le L_i(\bHrec(t-1), t) + \SRegBayesi{\bHrec(t-1)}{i}.
\end{equation}
The first inequality, $L(\bHrec(t-1), t+1) \le \Lmock(\bHrec(t-1), t)$, is derived from the fact that $L(\bH, t+1)$ is the minimum possible simple regret with $T-t$ samples.\footnote{Note that $\Lmock(\bHrec(t-1), t)$ skips round $t$ and imputes a sample by $\Urec(t)$.}
The second inequality is derived from the following argument. Since $\Lmock(\bHrec(t-1), t)$ simulates the sample from arm $i$ at round $t$ by $\Urec(t)$, then
\begin{align}
\Urec(t)&, \Irec(t\hspace{-0.1em}+\hspace{-0.1em}1), \Xrec(t\hspace{-0.1em}+\hspace{-0.1em}1), \Irec(t\hspace{-0.1em}+\hspace{-0.1em}2), \Xrec(t\hspace{-0.1em}+\hspace{-0.1em}2), \dots, \Irec(T), \Xrec(T) \text{~of~$\Lmock\hspace{-0.2em}(\bHrec(t\hspace{-0.1em}-\hspace{-0.1em}1),\hspace{-0.1em}t)$}
\intertext{and}
\\
\Xrec(t)&, \Irec(t\hspace{-0.1em}+\hspace{-0.1em}1), \Xrec(t\hspace{-0.1em}+\hspace{-0.1em}1), \Irec(t\hspace{-0.1em}+\hspace{-0.1em}2), \Xrec(t\hspace{-0.1em}+\hspace{-0.1em}2), \dots, \Irec(T), \Xrec(T) \text{\ of $L_i(\bHrec(t-1),t)$}
\end{align}
have an identical distribution.
Although $\Lmock(\bHrec(t-1), t)$ has an incorrect posterior on arm $i$, the mock-sample algorithm never recommends arm $i$, and the value
\begin{equation}
 \Lmock(\bHrec(t-1), t)  
\end{equation}
is exactly the same as the loss when we choose arms $\Irec(t), \Irec(t+1),\dots,\Irec(T)$ by using the recursive equation of $L_i(\bHrec(t-1), t)$ but $\Jrec(T) = \argmax_{j\in[K]\setminus\{i\}} \hatmurec_j(T)$ is recommended instead of $\argmax_{j\in[K]} \hatmurec_j(T)$.
Therefore, we have
\begin{align}
\Lmock(\bHrec(t-1), t) - L_i(\bH, t)
&\le \Ex_{\bHrec(t-1)}[\Ind[\argmax_j \hatmurec_j(T) = i, \ist(\bmu)=i] \SRegFreq(T) ] \\
&\le \Ex_{\bHrec(t-1)}[\Ind[\ist(\bmu)=i]\SRegFreq(T)].
\end{align}
\end{proof}

\begin{proof}[Proof of Lemma~\ref{lem_nodraw_one}]
Lemma~\ref{lem_nodraw_one} directly follows from Lemmas~\ref{lem_oneoptpost} and~\ref{lem_upper} with $i=3$.
\end{proof}

\section{Proof of Lemma~\ref{lem_nodraw_two}}
\label{sec_lower}

This section derives a lower bound for $\BI_i(\bHrec(t-1),t) := L(\bHrec(t-1), t+1) - L_i(\bHrec(t-1),t)$ for arms $i=1,2$.

\subsection{Reduction to a two-armed problem}
\label{subsec_reduction}

The following lemma bounds the gap between the three-armed best arm identification problem and the two-armed best arm identification problem in which one of the arms has been removed.
\begin{lem}{(Reduction to a two-armed problem)}\label{lem_exclude}
Let $\Hexrec{i}$ be the posterior of the two-armed instance in which arm $i$ has been removed.
The following inequality holds:
\begin{equation}\label{ineq_reduc}
|L(\bHrec(t-1), t+1) - L(\Hexrec{3}(t-1), t+1)| \le \SRegBayesi{\bHrec(t-1)}{3}.
\end{equation}
Similarly, we have 
\begin{equation}\label{ineq_reduc_i}
|L_j(\bHrec(t-1), t) - L_j(\Hexrec{3}(t-1), t)| \le \SRegBayesi{\bHrec(t-1)}{3}.
\end{equation}
for $j \in \{1,2\}$.
\end{lem}
\begin{proof}[Proof of Lemma~\ref{lem_exclude}]
For the three-armed instance with posterior $\bHrec(t-1)$, we may consider a recursive algorithm that behaves as if it were a two-armed instance and does not draw nor recommend arm $3$. For each sample path of rewards from arms $[3] \setminus \{3\}$, such an algorithm recommends exactly the same arm as the Bayes optimal algorithm on $\Hexrec{3}(t-1)$. 
The loss of such an algorithm, which is larger than $L(\bHrec(t-1),t)$, is at most $L(\Hexrec{3}(t-1), t) + \SRegBayesi{\bHrec(t-1)}{3}$ because the fraction of the simple regret such that arm $3$ is best is $\SRegBayesi{\bHrec(t-1)}{3}$. Therefore, we have
\begin{equation}\label{ineq_reduc_left}
L(\bHrec(t-1), t) - L(\Hexrec{3}(t-1), t) \le \SRegBayesi{\bHrec(t-1)}{3}.
\end{equation}

On the other hand, we may consider an algorithm for the two-armed best arm identification problem with posterior $\Hexrec{3}(t-1)$, which behaves as if it were facing the three-armed best arm identification problem with posterior $\bHrec(t-1)$, with an imaginary posterior of arm $3$. When the Bayes optimal algorithm on $\bHrec(t-1)$ would draw arm $3$, it uses the external randomizer to draw a pseudo-sample. The actual draw at that round can be any of the arms in $[3] \setminus \{3\}$, and the sample is simply discarded. For each sample path, it would recommend exactly the same arm if $\Jrec(T) \in \{1,2\}$, or could recommend any arm if $\Jrec(T)=3$, and the gap 
$L(\Hexrec{3}(t-1), t) - L(\bHrec(t-1), t)$ only appears when arm $3$ is the best.
Therefore, we have
\begin{equation}\label{ineq_reduc_right}
L(\Hexrec{3}(t-1), t) - L(\bHrec(t-1), t) \le \SRegBayesi{\bHrec(t-1)}{3}.
\end{equation}
Combining Eq.~\eqref{ineq_reduc_left} and \eqref{ineq_reduc_right} yields Eq.~\eqref{ineq_reduc}.
Eq.~\eqref{ineq_reduc_i} is derived by using the same approach.
\end{proof}

\subsection{Bounds on the two-armed problem\label{subsec_twobound}}

This subsection considers the loss of the two-armed best arm identification problem in which arm $3$ has been removed. 
The goal here is to show that  $\BI_i(\Hexrec{3}(t-1),t)$ is $\Omega(\poly(T^{-1}))$ if $\hDeltarec(t-1) := \hatmurec_1(t-1) - \hatmurec_2(t-1)$ is sufficiently small. To achieve this, we take the following steps.
First, Section~\ref{subsec_marloss} considers $\BI_i(\Hexrec{3}(T-1),T)$ at round $T$ (i.e., the last round); we show that $\BI_i(\Hexrec{3}(T-1),T)$ is $\Omega(\poly(T^{-1}))$ when $\hDeltarec(T-1)$ is small. %
Second, Section~\ref{subsec_opt_twogauss} shows that the Bayes optimal algorithm in the two-armed best arm identification problem draws the two arms alternately. This implies that $\Irec(t)$ for each $t$ is deterministic, which enables us to evaluate the distribution of $\Hexrec{3}(T-1)$ given $\Hexrec{3}(t-1)$. 
Third, Section~\ref{subsec_d2lb} shows that $\EW(t)$ implies that $\hDelta(t-1)$ is sufficiently small that $\BI_i(\Hexrec{3}(t-1),t)$ is $\Omega(\poly(T^{-1}))$.

\subsubsection{Marginal loss reduction}
\label{subsec_marloss}

First, the following lemma evaluates the simple regret given $\hDeltarec(T)$ (i.e., the simple regret after $T$ draws). 
\begin{lem}{\rm (Evaluation of the Bayesian simple regret after $T$ rounds)}\label{lem_zerostep}
Let $r$ be a zero-mean normal random variable with variance $\Vrest = 1/(\Nrec_1(T)) + 1/(\Nrec_2(T))$. Then, 
\begin{equation}
L(\Hexrec{3}(T), T+1) 
:=
\Ex_{\Hexrec{3}(T)}[\Delta_{\Jrec(T)}] 
= \int (r - |\hDeltarec(T)|)_+ \exp\left(-\frac{r^2}{2\Vrest}\right) dr. 
\end{equation}
\end{lem}
\begin{proof}
In the posterior $\Hexrec{3}(T)$, $\mu_i$, for $i=1,2$, are distributed as:
\begin{equation}
\mu_i \sim \Normal\left(\hatmurec_i(T), \frac{1}{\Nrec_i(T)}\right)
\end{equation}
and thus for $\{i,j\} \in \{1,2\}$,
\begin{equation}\label{ineq_difdist}
(\mu_i - \mu_j) - (\hatmurec_i(T) - \hatmurec_j(T)) \sim \Normal(0, \Vrest).
\end{equation}
Eq.~\eqref{ineq_difdist} implies that the quantity $\Vrest = 1/(\Nrec_1(T)) + 1/(\Nrec_2(T))$ is the variance of $\mu_1 - \mu_2$ in the posterior, which we call posterior variance.
For ease of discussion, assume that $\hatmurec_j(T) \ge \hatmurec_i(T)$.\footnote{It is straightforward to derive the same result when $\hatmurec_j(T) \le \hatmurec_i(T)$.}
By the definition of the Bayes optimal algorithm, $J(T) = \argmax_{k \in \{1,2\}} \hatmurec_k(T) = \hatmurec_j(T)$ and the simple regret is non-zero only when\footnote{That is, the signs of $\hatmurec_j(T) - \hatmurec_i(T)$ and $\mu_j - \mu_i$ are different.} $\mu_i > \mu_j$; namely,
\begin{align}
\Delta_{\Jrec(T)} =
\left(
\mu_i - \mu_j
\right)_+
\end{align}
and thus
\begin{align}
\Ex_{\Hexrec{3}(T)}[\Delta_{\Jrec(T)}] 
=\int (x - |\hDeltarec(T)|)_+ \exp\left(-\frac{x^2}{2\Vrest}\right) dx \text{\ \ \ (by Eq.~\eqref{ineq_difdist})}.
\end{align}
\end{proof}f

The following inequality quantifies the expected decrease in the loss function due to the last sample $i = \Irec(T)$.
\begin{lem}{\rm (One-step improvement with the $T$-th draw)}\label{lem_impsample}
The following inequality holds:
\begin{multline}
L(\Hexrec{3}(T-1), T+1)
-
L_i(\Hexrec{3}(T-1), T)
\\
>
\Prob\left[
|\hDeltarec(T-1)| \le l \le 2|\hDeltarec(T-1)|
\right]
\times \Prob\left[r \ge |\hDeltarec(T-1)|\right]
\times |\hDeltarec(T-1)| \\
\\ \ \ \ =: g_1(\hDeltarec(T-1), \Nrec_1(T-1), \Nrec_2(T-1), i) > 0.
\end{multline}
where $l$ and $r$ be zero-mean normal random variables with variances 
\begin{align}
\Vlast &= \frac{1}{\Nrec_i(T-1)(1+\Nrec_i(T-1))}\\
\Vrest &= \frac{1}{\Nrec_1(T)} + \frac{1}{\Nrec_2(T)} = \frac{1}{\Nrec_1(T-1)} + \frac{1}{\Nrec_2(T-1)} - \Vlast
\end{align}
respectively, and the probabilities are taken over $l$ and $r$.
In particular, if $T^{-2} \le \hDeltarec(T-1) \le 2T^{-2}$, then
\begin{equation}
g_1(\hDeltarec(T-1), \Nrec_1(T-1), \Nrec_2(T-1), i) = \Omega(T^{-6}).
\end{equation}
\end{lem}
\begin{proof}[Proof of Lemma~\ref{lem_impsample}]
By the update rule of the Bayesian posterior, we have
\begin{align}\label{ineq_bayesTtrans}
\hDeltarec(T) - \hDeltarec(T-1)  
\sim \Normal\left(0, \Vlast\right).
\end{align}

On the one hand,
\[
L(\Hexrec{3}(T-1), T+1) = \Ex_{\Hexrec{3}(T-1)}[\Delta_{\Jrec(T-1)}]
\]
is the loss when we recommend $\Jrec(T-1)$ immediately after the round $T-1$. On the other hand, 
\[
L_i(\Hexrec{3}(T-1), T)
\]
is the loss when we sample $\Irec(T)=i$ and chooses recommendation arm $\Jrec(T)$ after we update $\hDeltarec(T-1)$.
The sample $\Irec(T)=i$ reduces the posterior variance from $\Vrest + \Vlast$ to $\Vrest$.

For ease of discussion,\footnote{Deriving the same result in the case of $\hDeltarec(T-1) \le 0$ is straightforward.} assume that $\hDeltarec(T-1) \ge 0$.
Letting $\Phi_r(r)$ and $\Phi_l(l)$ be the corresponding cumulative distribution functions of zero-mean normal variables with their variances $\Vrest$ and $\Vlast$, we have
\begin{align} %
\lefteqn{
L(\Hexrec{3}(T-1), T+1)
-
L_i(\Hexrec{3}(T-1), T)
}\\
&=
\int \int ((r+l) - \hDeltarec(T-1))_+ %
d\Phi_r(x) d\Phi_l(z)\\
&\ \ \ \ \ \ \ \ \ \ \ -
\int \int 
(r - |\hDeltarec(T-1)-l|)_+ 
d\Phi_r(x) d\Phi_l(z)\\
&\text{\ \ \ (by Lemma~\ref{lem_zerostep} and Eq.~\eqref{ineq_bayesTtrans})}\\
&\ge 
\Prob\left[
\hDeltarec(T-1) \le l \le 2\hDeltarec(T-1)
\right]
\times \Prob[r \ge \hDeltarec(T-1)]
\times \hDeltarec(T-1) > 0. \label{ineq_impsample_core}
\end{align}

We next evaluate Eq.~\eqref{ineq_impsample_core} when $T^{-1} \le \hDeltarec(T-1) \le 2T^{-1}$. 
The fact that $\Nrec_i(T-1) < T$ implies that $\Vrest, \Vlast \ge 1/T^2$, and thus 
\begin{align}
\lefteqn{
\Prob\left[
\hDeltarec(T-1) \le l \le 2\hDeltarec(T-1)
\right]
}\\
&:= \frac{1}{\sqrt{2\pi \Vlast}} \int 
\Ind[\hDeltarec(T-1) \le l \le 2\hDeltarec(T-1)]
\exp\left(-\frac{l^2}{2\Vlast}\right) dl\\
&\ge \frac{1}{\sqrt{2\pi}} \int 
\Ind[\hDeltarec(T-1) \le l \le 2\hDeltarec(T-1)]
\exp\left(-\frac{T^2 l^2}{2}\right) dl\\
&\ge \frac{1}{\sqrt{2\pi}}
\times
\frac{1}{T^2}
\times
\exp\left(-T^2\times \left(\frac{1}{T^2}\right)^2 \times \frac{1}{2} \right)\\
&= \Omega(T^{-2}).
\end{align}
It is straightforward to see that 
\[
\Prob[r \ge \hDeltarec(T-1)] = \Omega(T^{-2})
\]
and thus Eq.~\eqref{ineq_impsample_core} is $\Omega(T^{-2}\times T^{-2}\times T^{-2})
= \Omega(T^{-6})$.
\end{proof} %

The following lemma states that the optimal strategy at round $T$ is to draw the arm given by $\argmin_i \Nrec_i(T-1)$.
\begin{lem}{\rm (Optimal action at round $T$)}\label{lem_smldraf}
Let $i = \argmin_{k \in \{1,2\}} \Nrec_k(T-1)$ and $j\ne i$ be such that $\Nrec_j(T-1) < \Nrec_i(T-1)$.\footnote{Notice that we have $\Nrec_1(T-1)+\Nrec_2(T-1) = 2T'-1$ and thus $\Nrec_i(T-1) \ge T', \Nrec_j(T-1) \le T'-1$.}
Then, 
\begin{equation}%
L_i(\Hexrec{3}(T-1), T) - L_j(\Hexrec{3}(T-1), T)
\ge 
g_2(\hDeltarec(T-1), \Nrec_1(T-1), \Nrec_2(T-1))
\end{equation}
for some function $g_2(\hDeltarec(T-1), \Nrec_1(T-1), \Nrec_2(T-1)) > 0$.
In particular, if $T^{-2} \le \hDeltarec(T-1) \le 2T^{-2}$, then
\begin{equation}
g_2(\hDeltarec(T-1), \Nrec_1(T-1), \Nrec_2(T-1)) = \Omega(T^{-8}).
\end{equation}
\end{lem}
\begin{proof}[Proof of Lemma~\ref{lem_smldraf}]
Let $V(N) = 1/(N(N+1))$. 
By assumption, $\Nrec_i(T-1) > \Nrec_j(T-1)$ and thus $V(\Nrec_i(T-1)) < V(\Nrec_j(T-1))$.
By the update rule of the Bayesian posterior, if $\Irec(T) = i$, then we have
\begin{align} %
\hDeltarec(T) - \hDeltarec(T-1) \sim
\Normal\left(0, V(\Nrec_i(T-1))\right). \label{ineq_vijvar}
\end{align}
Moreover, the distribution of a normal variable with variance $V(\Nrec_j(T-1))$ 
is the same as the distribution of two independent normal variables with variances $V(\Nrec_i(T-1))$ and $V(\Nrec_j(T-1))-V(\Nrec_i(T-1))$, and so, if $\Irec(T)=j$, then 
\begin{align}
\hDeltarec(T) - \hDeltarec(T-1) 
&\sim \Normal\left(0, V(\Nrec_j(T-1))\right)\\
&\sim \Normal\left(0, V(\Nrec_i(T-1))\right) + \Normal\left(0, V(\Nrec_j(T-1))-V(\Nrec_i(T-1))\right). \label{ineq_vijvartwo}
\end{align}
Let
\begin{align}
r &\sim \Normal(0, \Vrest), \text{\ where $\Vrest = \frac{1}{N_i(T-1)} + \frac{1}{N_j(T-1)+1}$},\\
l &\sim \Normal(0, \Vlast), \text{\ where $\Vlast = V(N_i(T-1))$},\\
g &\sim \Normal(0, \Vgap), \text{\ where $\Vgap = V(N_j(T-1))-V(N_i(T-1))$}
\end{align}
be three independent zero-mean normal random variables. Note also that the posterior variance is
\begin{align} 
\frac{1}{\Nrec_1(T)} + \frac{1}{\Nrec_2(T)}
&=
\left\{
\begin{array}{ll} 
\Vrest + \Vgap& \text{(if $I(T) = i$)} 
, \\
\Vrest
& \text{(otherwise, if $I(T) = j$)}.\label{ineq_vijposteriorvar}
\end{array}
\right. 
\end{align}
Then,
\begin{align} %
\lefteqn{
L_i(\Hexrec{3}(T-1), T) - L_j(\Hexrec{3}(T-1), T)
}\\
&=
\int \int \int ((r+g) - |\hDeltarec(T-1)-l|)_+ %
d\Phi_r(r) d\Phi_l(l) d\Phi_g(g)
\\
&\ \ \ \ \ \ -
\int \int \int 
(r - |\hDeltarec(T-1)-l-g|)_+ 
d\Phi_r(r) d\Phi_l(l) d\Phi_g(g)\\ %
&\text{\ \ \ (by Eq.~\eqref{ineq_vijvar}, \eqref{ineq_vijvartwo}, \eqref{ineq_vijposteriorvar}, and Lemma~\ref{lem_zerostep})}\\
&\ge 
\Prob\left[
|\hDeltarec(T-1)-l| \le g \le 2|\hDeltarec(T-1)-l|
\right]\\
&\ \ \ \ \ \ 
\times \Prob[r \ge |\hDeltarec(T-1)-l|]
\times |\hDeltarec(T-1)-l|
\\
&\text{\ \ \ (by the same discussion as Eq.~\eqref{ineq_impsample_core})}\\
&\ge 
\Prob\left[
0 \le l \le \frac{1}{3}|\hDeltarec(T-1)|
\right]
\times
\Prob\left[
|\hDeltarec(T-1)| \le g \le \frac{4}{3}|\hDeltarec(T-1)|
\right]\\
&\ \ \ \ 
\times \Prob[r \ge |\hDeltarec(T-1)|]
\times \frac{2}{3}|\hDeltarec(T-1)|\\
&=: g_2(\hDeltarec(T-1), \Nrec_1(T-1), \Nrec_2(T-1)).
\end{align}

We next evaluate $g_2$ for $T^{-2} \le \hDeltarec(T-1) \le 2T^{-2}$. By $\Nrec_1(T), \Nrec_2(T) \le T$ and $\Nrec_1(T) \ne \Nrec_2(T)$, it is straightforward to confirm that $\Vrest,\Vlast,\Vgap \ge 1/T^3$. Therefore,
\begin{align}
\lefteqn{
g_2(\hDeltarec(T-1), \Nrec_1(T-1), \Nrec_2(T-1))
}\\
&=
\Prob\left[
0 \le l \le \frac{1}{3}|\hDeltarec(T-1)|
\right]
\times\\
&\ \ \ \ \ \ \ \ 
\Prob\left[
|\hDeltarec(T-1)| \le g \le \frac{4}{3}|\hDeltarec(T-1)|
\right]
\times \Prob[r \ge |\hDeltarec(T-1)|]
\times \frac{2}{3}|\hDeltarec(T-1)|.
\end{align}
Here,
\begin{align}
\Prob\left[
0 \le l \le \frac{1}{3}|\hDeltarec(T-1)|
\right]
&= \frac{1}{\sqrt{2\pi \Vlast}} \int \Ind[0 \le l \le \frac{1}{3}|\hDeltarec(T-1)|] \exp\left(-\frac{l^2}{2\Vlast}\right) dl\\
&\ge \frac{1}{\sqrt{2\pi}} \int \Ind[0 \le l \le \frac{1}{3}|\hDeltarec(T-1)|] \exp\left(-\frac{T^3 l^2}{2}\right) dl\\
&\label{ineq_precision}\\
&\ge \frac{1}{\sqrt{2\pi}} \int \Ind\left[0 \le l \le \frac{1}{3}|\hDeltarec(T-1)|\right] \exp\left(-\frac{T^3 l^2}{2}\right) dl\\
&\ge \frac{1}{\sqrt{2\pi}}
\times
\frac{1}{3T^{2}}
\times
\exp\left(-\frac{T^3}{2}\times \left(\frac{2}{3T^2}\right)^2
\right)\\
&= \Omega(T^{-2})
\end{align}
and it is easy to show that 
\begin{equation}
\Prob\left[
|\hDeltarec(T-1)| \le g \le \frac{4}{3}|\hDeltarec(T-1)|
\right]
\end{equation}
and
\[
 \Prob[r \ge |\hDeltarec(T-1)|]
\]
are also $\Omega(T^{-2})$. 
In summary, $g_2(\hDeltarec(T-1), \Nrec_1(T-1), \Nrec_2(T-1))$ is an $\Omega(T^{-8})$ term if $T^{-2} \le \hDeltarec(T-1) \le 2T^{-2}$.
\end{proof} %

\subsubsection{Characterization of the two-armed problem}
\label{subsec_opt_twogauss}

in Section~\ref{subsec_marloss}, the reduction of the loss by the last sample was quantified. 
Based on this result, the following lemma exactly characterizes a Bayes optimal algorithm for the two-armed best arm identification problem.
\begin{lem}%
\label{lem_drawnumdiff}
Let $K=2$. Then, for all $s$, we have that $i = \argmin_k \Nrec_k(s-1)$ minimizes $L_i(\Hexrec{3}(s-1), s)$.
\end{lem}
\begin{proof}[Proof of Lemma~\ref{lem_drawnumdiff}]
The proof is by induction.
Namely, assuming that $L(\Hexrec{3}(s), s+1)$ is the simple regret of the algorithm for which $I(s') = \argmin_i \Nrec_i(s'-1)$ for $s'=s+1,s+2,\dots,T$, we prove that $L(\Hexrec{3}(s-1), s)$ is the simple regret regret of the algorithm for which $I(s') = \argmin_i \Nrec_i(s'-1)$ for $s'=s,s+1,s+2,\dots,T$.

Let $S = T-s+1$ be the remaining budget.
The only case that we need to consider\footnote{In other cases, the induction step is trivial.} is $S \le |\Nrec_1(s-1) - \Nrec_2(s-1)|$. Without loss of generality, we assume $\Nrec_1(s-1) \le \Nrec_2(s-1) - S$. 
In the following, we compare (A)~the simple regret of an algorithm that draws $S$ samples of arm $1$ and (B)~the simple regret of an algorithm that draws $S-1$ samples from arm $1$ and a sample from arm $2$. We show that the former has a lower expected simple regret. The order in which arms are drawn does not matter for a deterministic sequence of draws, and thus it suffices to show that
\begin{align}
\lefteqn{
\Ex\left[
\Delta_{J(T)}
|\Hexrec{3}(s-1), \Irec(s)=1,\Irec(s+1)=1,\dots,\Irec(T-1)=1,\Irec(T)=1
\right]
}\\
&\ \ \ <
\Ex\left[
\Delta_{J(T)}
|\Hexrec{3}(s-1), \Irec(s)=1,\Irec(s+1)=1,\dots,\Irec(T-1)=1, \Irec(T)=2
\right]\\
&
\label{ineq_greedydiff}
\end{align}
in what follows. Note that the RHS and LHS of Eq.\eqref{ineq_greedydiff} have the identical distribution of $\Hexrec{3}(T-1)$ because $\Irec(s)=1,\Irec(s+1)=1,\dots,\Irec(T-1)=1$ are the same. Therefore, we have,
\begin{align}
\lefteqn{
\text{RHS of Eq.~\eqref{ineq_greedydiff}}
-
\text{LHS of Eq.~\eqref{ineq_greedydiff})}
}\\
&=
\Ex
\Biggl[
\left(
L_2(\Hexrec{3}(T-1), T) - L_1(\Hexrec{3}(T-1), T)
\right)
\\
&\ \ \ \ \ \ \ \ \ \ \ \ \ \ \ \ \ \ \ \ \ \biggl|
\Irec(s)=1,\Irec(s+1)=1,\dots,\Irec(T-1)=1
\Biggr]
\\
&> 0,
\end{align}
since Lemma~\ref{lem_smldraf} states that $L_2(\Hexrec{3}(T-1), T) - L_1(\Hexrec{3}(T-1), T) > 0$.
\end{proof} %

\subsubsection{A lower bound for the EBI in the two-armed problem \label{subsec_d2lb}}

Section~\ref{subsec_reduction} reduced the three-armed instance with posterior $\bHrec(t-1)$ to the two-armed instance with posterior $\Hexrec{3}(t-1)$. Section~\ref{subsec_opt_twogauss} showed that the Bayes optimal algorithm in the two-armed case draws arms with small $N_i(t-1)$. Given these results, we bound the EBI from below in the two-armed best arm identification problem.
Assuming event $\EW(t)$, the following lemma provides a lower bound for the probability that $\hDeltarec(T-1)$ lies in a small region around zero.
\begin{lem}\label{lem_twoarm_driftclose}
Assume that $\Irec(s) = \argmin_{i \in \{1,2\}} \Nrec_i(s-1)$ for $s=t,t+1,\dots,T$. Let $\bHrec(t-1) = \bH(t-1)$ be any posterior such that $\EW(t)$ holds, and let $\hDelta(t-1)$, $N_1(t-1)$, and $N_2(t-1)$ be the statistics determined by $\bH(t-1)$. 
Then, 
\begin{equation}
\Prob
\Bigl[
\frac{1}{T^2} \le \hDeltarec(T-1) \le \frac{2}{T^2}
\Bigl| \hDelta(t-1), N_1(t-1), N_2(t-1)
\Bigr]
\ge 
g_3
\end{equation}
where $g_3$ is an $\Omega(T^{-110})$ function.\footnote{The function $g_3$ is explicitly defined in the last stage of the proof.}
\end{lem}
\begin{proof}[Proof of Lemma~\ref{lem_twoarm_driftclose}]
For ease of discussion, ties are broken such that $I(s)=1$. 
Since $I(s) = \argmin_{i \in \{1,2\}} \Nrec_i(s-1)$ for $s=t,t+1,\dots,T$, we have 
\begin{align}
M_1
&:=
\Nrec_1(T-1) - \Nrec_1(t-1) 
=
\left\{
\begin{array}{ll} 
(T' - N_1(t-1)) & \text{(if $N_1(t-1)\le T'$)}, \\
0
& \text{(otherwise)},
\end{array}
\right. \\
M_2
&:= 
\Nrec_2(T-1) - \Nrec_2(t-1) =
\left\{
\begin{array}{ll} 
(T' - N_2(t-1)-1) & \text{(if $N_2(t-1) \le T'-1$)}, \\
0
& \text{(otherwise)}.
\end{array}
\right.\\
\label{ineq_sample_sdif}
\end{align}
Note that $M_1$ and $M_2$ are deterministic given $N_1(t-1)$ and $N_2(t-1)$.
Eq.~\eqref{ineq_sample_sdif} implies that 
\begin{align}
\hDeltarec(T-1) - \hDelta(t-1) &\sim \Normal\left(0, \Vtrans\right),%
\end{align}
where 
\begin{align}\label{ineq_vtransdist}
\Vtrans &:= \frac{1}{N_1(t-1)} - \frac{1}{N_1(t-1)+M_1} + \frac{1}{N_2(t-1)} - \frac{1}{N_2(t-1)+M_2}.%
\end{align}
By $\EW(t)$, we have
\begin{align}
\lefteqn{
|\hDelta(t-1)| 
}\\
&\le |\hatmu_{1,N_1(t-1)}-\hatmu_{1,\Nrec_1(T-1)}| 
     + |\hatmu_{2,N_2(t-1)}-\hatmu_{2,\Nrec_2(T-1)}| 
     + |\hDeltarec(T-1)|\\
&\le 3\sqrt{\log T}
\sqrt{
\sum_{m=N_1(t-1)}^{\Nrec_1(T-1)}
\frac{1}{m^2}
}
+ 3\sqrt{\log T}
\sqrt{
\sum_{m=N_2(t-1)}^{\Nrec_2(T-1)}
\frac{1}{m^2}
}
+ \frac{1}{T^2}\\
&\text{\ \ \ \ (by $\EY_1, \EZ_1, \EY_2, \EZ_2$)}\\
&\le 3\sqrt{2\log T} 
\sqrt{
\sum_{m=N_1(t-1)}^{\Nrec_1(T-1)}
\frac{1}{m(m+1)}
}
+ 3\sqrt{2\log T}
\sqrt{
\sum_{m=N_2(t-1)}^{\Nrec_2(T-1)}
\frac{1}{m(m+1)}
}
+ \frac{1}{T^2}\\
&\le 3\sqrt{2\log T} 
\sqrt{
\frac{1}{N_1(t-1)} - \frac{1}{\Nrec_1(T-1)}
}
+
\sqrt{
\frac{1}{N_2(t-1)} - \frac{1}{\Nrec_2(T-1)}
}
+ \frac{1}{T^2}\\
&\le
3\sqrt{2\log T} 
\sqrt{
3\left(
\frac{1}{N_1(t-1)} - \frac{1}{\Nrec_1(T-1)}
+
\frac{1}{N_2(t-1)} - \frac{1}{\Nrec_2(T-1)}
+ \frac{1}{T^4}
\right)
}\\
&\text{\ \ \ \ (by $\sqrt{x}+\sqrt{y}+\sqrt{z} \le \sqrt{3(x+y+z)}$)},
\end{align} 
and thus
\begin{align}
(\hDelta(t-1))^2 
&\le 54\log T \left(\Vtrans + \frac{1}{T^4}\right)\\
&\le 108 (\log T) \Vtrans.
\label{ineq_hdeltadriftbound}
\end{align}
We then have
\begin{align}\label{ineq_tinv_inprob}
\lefteqn{
\Prob\Bigl[
\frac{1}{T^2} \le \hDeltarec(T-1) \le \frac{2}{T^2}
\Bigl| \hDelta(t-1), N_1(t-1), N_2(t-1)\Bigr]
}\\
&=
\frac{1}{\sqrt{2\pi \Vtrans}}
\int 
\Ind\left[s \in [\hDelta(t-1)+\frac{1}{T^2}, \hDelta(t-1)+\frac{2}{T^2}]\right]
\exp\left(- \frac{s^2}{2\Vtrans} \right) ds\\
&\text{\ \ \ \ (by Eq.~\eqref{ineq_vtransdist})}\\
&\ge
\frac{1}{\sqrt{2\pi}}
\int 
\Ind\left[s \in [\hDelta(t-1)+\frac{1}{T^2}, \hDelta(t-1)+\frac{2}{T^2}]\right]
\exp\left(- \frac{s^2}{2\Vtrans} \right) ds\\
&\ge
\frac{1}{\sqrt{2\pi}}
\int 
\Ind\left[s \in [\hDelta(t-1)+\frac{1}{T^2}, \hDelta(t-1)+\frac{2}{T^2}]\right]\\
&\ \ \ \ \ \ \ \ \ 
\times \exp\left(- \frac{(\hDelta(t-1))^2}{\Vtrans} \right) \exp\left(- \frac{4}{\Vtrans T^4} \right) ds\\
&\text{\ \ \ (by $(x+y)^2 \le 2x^2 + 2y^2$)}\\
&=
\frac{1}{T^2\sqrt{2\pi}}
\exp\left(- \frac{(\hDelta(t-1))^2}{\Vtrans} \right) \exp\left(- \frac{4}{\Vtrans T^4} \right)\\
&=
\frac{1}{T^2\sqrt{2\pi}}
\exp\left(- \frac{(\hDelta(t-1))^2}{\Vtrans} \right) e^{-4}\\
&\text{\ \ \ (by $\Vtrans \ge T^{-2} \ge T^{-4}$)}\\
&\ge
\frac{e^{-4}}{T^2\sqrt{2\pi}}
\exp\left(- 
\frac{108(\log T) \Vtrans}{\Vtrans}
\right)
\\
&\text{\ \ \ \ (by Eq.~\eqref{ineq_hdeltadriftbound})}\\
&\ge
\frac{e^{-4}}{T^2\sqrt{2\pi}}
\exp\left(- 
108\log T
\right)\\
&=
\frac{e^{-4}}{T^2\sqrt{2\pi}}
\exp\left(- 
108 \log T
\right)
= \frac{e^{-4}}{\sqrt{2\pi}}T^{-110} =: g_3(T).
\end{align}
\end{proof} %

Based on the previous lemmas, we can bound $\BI_i(\Hexrec{3}(t-1), t)$. 
\begin{lem}\label{lem_twoarm_di}
Assume that $\EW(t)$ holds.
For $i \in \{1,2\}$, we have
\begin{equation}\label{ineq_dlower}
\BI_i(\Hexrec{3}(t-1), t) \ge g_1 g_3.
\end{equation}
Moreover, if $N_i(t-1) < T'$ and $ N_j(t-1) \ge T'$, then 
\begin{equation}\label{ineq_dlowersmlpref}
\BI_i(\Hexrec{3}(t-1), t) - \BI_j(\Hexrec{3}(t-1), t) 
\ge g_2 g_3.
\end{equation}
\end{lem}
\begin{proof}

We first derive Eq.~\eqref{ineq_dlower}. 
By definition,
\begin{align}
\BI_i(\Hexrec{3}(t-1), t) := L(\Hexrec{3}(t-1), t+1) - L_i(\Hexrec{3}(t-1), t).
\end{align}
Lemma~\ref{lem_drawnumdiff} states that the difference between $L(\Hexrec{3}(t-1), t+1)$ and $L_i(\Hexrec{3}(t-1), t)$ is derived by the additional sample $\Irec(T)$ after a deterministic sequence of draws $\Irec(t),\Irec(t+1),\dots,\Irec(T-1)$. Therefore,
\begin{align}
\lefteqn{
L(\Hexrec{3}(t-1), t+1) - L_i(\Hexrec{3}(t-1), t)
}\\
&\ge
\Prob\Bigl[
\frac{1}{T^2} \le \hDeltarec(T-1) \le \frac{2}{T^2}
\Bigl| \hDelta(t-1), N_1(t-1), N_2(t-1)
\Bigr]
\times g_1\\
&\text{\ \ \ (by Lemma~\ref{lem_impsample} and the discussion above)}\\
&\ge
g_3 \times g_1.\\
&\text{\ \ \ (by Lemma~\ref{lem_twoarm_driftclose})}\\
\end{align}

We next derive Eq.~\eqref{ineq_dlowersmlpref}.
Notice that under $N_i(t-1) \ge T'$, drawing arm $i$ results in $\Nrec_i(T) = T'+1$ and thus $\Nrec_j(T) = T'-1$, which is suboptimal.
Lemma~\ref{lem_drawnumdiff} states that the gap between $\BI_i(\Hexrec{3}(t-1), t)$ and $\BI_j(\Hexrec{3}(t-1), t)$ is derived by the additional sample $\Irec(T)=i $ and $j$ after a deterministic sequence of draws of $\Irec(t)=j,\Irec(t+1)=j,\dots,\Irec(T-1)=j$. 
From the discussion above, we have
\begin{align}
\lefteqn{
L_i(\Hexrec{3}(t-1), t+1) - L_j(\Hexrec{3}(t-1), t)
}\\
&\ge
\Prob\Bigl[
\frac{1}{T^2} \le \hDeltarec(T-1) \le \frac{2}{T^2}
\Bigl| \hDelta(t-1), N_1(t-1), N_2(t-1)
\Bigr]
\times g_2\\
&\text{\ \ \ (by Lemma~\ref{lem_smldraf} and the discussion above)}\\
&\ge
g_3 \times g_2.\\
&\text{\ \ \ (by Lemma~\ref{lem_twoarm_driftclose})}\\
\end{align}
\end{proof}

\subsection{Proof of Lemma~\ref{lem_nodraw_two}\label{subsec_lower_final}}

\begin{proof}[Proof of Lemma~\ref{lem_nodraw_two}]
For $i \in \{1,2\}$, we have %
\begin{align}
\BI_i(\bHrec(t-1), t) 
&:= L(\bHrec(t-1), t+1) - L_i(\bHrec(t-1), t)\\
&\ge L(\Hexrec{3}(t-1), t+1) - L_i(\Hexrec{3}(t-1), t) - 2 \SRegBayesi{\bHrec(t-1)}{3}\\
&\text{\ \ \ (by Lemma~\ref{lem_exclude})}\\
&\ge g1 g3 - 2 \SRegBayesi{\bHrec(t-1)}{3}\\
&\text{\ \ \ (by Lemmas~\ref{lem_oneoptpost} and~\ref{lem_twoarm_di})}\\ 
&\ge g1 g3 - 2 \gthree \\ %
&= \Omega(T^{-110}) - O\left(\frac{(C_U-6)^2}{4}\right) = \Omega(T^{-110}). \text{\ \ \ \ (by Eq.~\eqref{ineq_cularge})}
\end{align}

Moreover, letting $\Nrec_j(T)=T'$ and $i \ne j$, we have
\begin{align}
\lefteqn{
\BI_i(\bHrec(t-1), t)
-
\BI_j(\bHrec(t-1), t)
}\\
&\ge
\BI_i(\Hexrec{3}(t-1), t)
-
\BI_j(\Hexrec{3}(t-1), t)
- 4 \SRegBayesi{\bHrec(t-1)}{3}\\
&\text{\ \ \ (by Lemma~\ref{lem_exclude})}\\
&\ge g2 g3 - 4 \gthree\\ &\text{\ \ \ (by Lemmas~\ref{lem_oneoptpost} and~\ref{lem_twoarm_di})}\\
&= \Omega(T^{-110}) - O\left(\frac{(C_U-6)^2}{4}\right) = \Omega(T^{-110}). \text{\ \ \ \ (by Eq.~\eqref{ineq_cularge})}
\end{align}
\end{proof} %

\section{Tightness of the Bound}
\label{sec_tight}

The goal of Theorem~\ref{thm_reglowinst} was to derive a polynomial simple regret of the Bayes optimal algorithm, and the orders of the polynomials are certainly not tight. Indeed, we can replace 
\[
\EY_i := \left\{|\hatmu_{i,T'}| \le 
\frac{1}{T^2}
\right\}
\]
with a slightly larger bound of
\[
\left\{|\hatmu_{i,T'}| \le 
\frac{1}{T^{3/2}}
\right\}.
\]
The place that requires the precision of $T^{-3/2}$ is Eq.~\eqref{ineq_precision} in which the exponent of the term $\exp\left(-\frac{T^3 l^2}{2}\right)$ is required to be $O(1)$.\footnote{This part bounds $|L_i(\Hexrec{3}(T-1), T) 
-
L_j(\Hexrec{3}(T-1), T)|$ from below. This value is decomposed into several components of $\poly(T^{-1})$. In particular, Eq.~\eqref{ineq_precision} corresponds to the probability that a normal random variable with variance $T^{-3}$ is larger than a given $\hatmurec_i(T-1) - \hatmurec_j(T-1)$. For this probability to be $\poly(T^{-1})$, we require $|\hatmurec_i(T-1) - \hatmurec_j(T-1)|= O(T^{-3/2})$.} In this paper, we adopt a bound of $T^{-2}$ for ease of exposition.

\end{document}